\documentclass[3p,preprint]{IEEEtran}

\usepackage[utf8]{inputenc} 
\usepackage[T1]{fontenc}    
\usepackage{hyperref}       
\usepackage{url}            
\usepackage{booktabs}       
\usepackage{amsfonts}       
\usepackage{nicefrac}       
\usepackage{microtype}      
\usepackage{amsmath}
\usepackage{amsthm}
\usepackage{color,soul}
\usepackage{amssymb}
\usepackage{graphicx}
\usepackage{epstopdf}
\usepackage{subfigure}
\usepackage[normalem]{ulem}
\usepackage{booktabs}
\usepackage{threeparttable}
\usepackage{subfigure}
\usepackage{cleveref}
\usepackage[justification=centering]{caption}
\usepackage{cite}

\DeclareMathOperator*{\argmin}{argmin}

\sethlcolor{yellow}

\theoremstyle{plain}
\newtheorem{theorem}{Theorem}[]
\newtheorem{lemma}[]{Lemma}

\newtheorem*{cor}{Corollary}

\usepackage[figuresright]{rotating}

\begin{document}

\title{Efficacy of regularized multi-task learning based on SVM models}



\author{Shaohan~Chen,~Zhou~Fang,~Sijie~Lu,~and~Chuanhou~Gao,~\IEEEmembership{Senior Member,~IEEE}
\thanks{
This work was supported by the National Natural Science Foundation of China under grants 62111530247 and 12071428, and the Zhejiang Provincial Natural Science Foundation of China under grant LZ20A010002.}
\thanks{S. Chen, S. Lu and C. Gao are with the School of Mathematical Sciences, Zhejiang
University, Hangzhou 310027, China (Corresponding e-mail: gaochou@zju.edu.cn (C. Gao)).}
\thanks{Z. Fang is with the Department of Biosystems Science and Engineering, ETH Zurich, Switzerland.}}

\IEEEtitleabstractindextext{
\begin{abstract}
This paper investigates the efficacy of a regularized multi-task learning (MTL) framework based on SVM (M-SVM) to answer whether MTL always provides reliable results and how MTL outperforms independent learning. We first find that M-SVM is Bayes risk consistent in the limit of large sample size. This implies that despite the task dissimilarities, M-SVM always produces a reliable decision rule for each task in terms of misclassification error when the data size is large enough. 
Furthermore, we find that the task-interaction vanishes as the data size goes to infinity, and the convergence rates of M-SVM and its single-task counterpart have the same upper bound.
The former suggests that M-SVM cannot improve the limit classifier's performance; based on the latter, we conjecture that the optimal convergence rate is not improved when the task number is fixed. As a novel insight of MTL, our theoretical and experimental results achieved an excellent agreement that the benefit of the MTL methods lies in the improvement of the pre-convergence-rate factor (PCR, to be denoted in Section III) rather than the convergence rate. Moreover, this improvement of PCR factors is more significant when the data size is small. 

\end{abstract}

}

\maketitle

\IEEEdisplaynontitleabstractindextext

\IEEEpeerreviewmaketitle

\ifCLASSOPTIONcompsoc
\IEEEraisesectionheading{\section{Introduction}\label{sec:introduction}}
\else
\section{Introduction}
\label{sec:introduction}
\fi

\IEEEPARstart{R}{ecent} years have seen a considerable success of machine learning technologies, including deep neural networks (DNNs), SVMs, decision trees, etc. These methods usually work well when the training and test data are drawn from the same distribution and ``big data'' is available (e.g., DNNs). 
However, the conditions of the same distribution and ``big data'' cannot always hold in many real-world problems—for example, building self-driving systems and modeling rare diseases in healthcare records. 
Thus, developing methods to learn multiple tasks simultaneously and efficiently is significant. In this case, multi-task learning (MTL) \cite{caruana1998multitask} provides a general way to accommodate these situations. {A vast of elaborate MTL methods have been designed in various settings to solve specific problems \cite{Chen2018,Li2018,Wang2019,wang2018sparse,gu2014multitask,jiang2015multi,jiang2016novel,evgeniou2005learning,yu2005learning,zhang2021survey}}. 

{Also, there have been a lot of notable theoretical justifications for the superiority of MTL over independent learning. Ando et al. \cite{ando2005framework} studied MTL based on structural learning and showed that it estimates the shared hypothesis space more reliably if the number of tasks $T$ is large enough. Baxter \cite{baxter1995learning,baxter2000model} investigated MTL in the framework of bias learning and showed that it significantly reduces the sampling burden for good generalization on novel tasks if $T$ is arbitrarily large. He also illustrated that MTL (within a Bayesian context) efficiently decays the information required to learn each task as $T$ grows considerably \cite{baxter1997bayesian}. Maurer et. al. \cite{maurer2016benefit} discussed MTL in representation learning and showed that it vanishes the cost to learn the representation in the multi-task limit ($T\rightarrow \infty$). Liu et. al. \cite{liu2017algorithm} analyzed MTL using the sample average stability measure and showed that it estimates the shared parameter more accurately when $T$ is large enough. Zhao et. al. \cite{zhang2020generalization} viewed MTL as a vector-valued function learning problem and showed that it requires a smaller sample size of each task to achieve the same performance.}


{The existing theoretical analysis of MTL relies on the pursuit of tighter error bounds, which is not enough to guarantee the superiority of MTL. 
Moreover, as introduced above, 
the advantages of MTL can so far only be seen when $T$ is sufficiently large, which makes MTL challenging to be trusted by users. The arbitrarily large $T$ requirement is rarely satisfied in real applications, e.g., it is often difficult to access a dataset with large $T$ in the healthcare and industrial systems. Therefore, it is crucial to investigate the intrinsic benefits of MTL by applying a more elaborate analysis method to it when $T$ is fixed.}

This paper moves the steps forward and tries to explore the essential benefits of MTL with respect to the misclassification error when $T$ is fixed. {Particularly, we concern with a popular regularized multi-task SVM (M-SVM) proposed by \cite{evgeniou2004regularized}, which provides an avenue to share useful knowledge among tasks through parameters. }
Note that the analysis of the misclassification error has been intensively researched in the single-task learning contexts \cite{wu2006analysis,chen2004support,ying2006online,ying2007learnability,lin2016iterative,fan2017learning}. {This work solved two main technical difficulties of adapting the misclassification error analysis to MTL: considering the interactions between tasks carefully and estimating the additional error caused by the randomness of the sampling frequencies for each task (see \Cref{section 3}).}

We first analyze the asymptotic performance of M-SVM to show that it is Bayes risk consistent in the limit of large sample size. This implies that despite the task dissimilarities, M-SVM always produces a reliable decision rule for each task in terms of the misclassification error when the data size is large enough. 
Second, we find that the task-interaction vanishes as the data size goes to infinity, and the convergence rates of M-SVM and its single-task counterpart have the same upper bound. The former suggests that M-SVM cannot improve the limit classifier's performance; based on the latter, we conjecture that the optimal convergence rate is not improved when the number of tasks is fixed. The intuition of this conjecture is that the generalization capacity mainly depends on the hypothesis space's complexity which MTL cannot change. 

{As a novel insight of MTL, our theoretical (on a one-dimensional classification problem) and experimental results achieved an excellent agreement that the benefit of M-SVM lies in improving the pre-convergence-rate (PCR) factor (which is denoted as the ratio between the excess misclassification error and its real convergence rate in Section III) rather than the convergence rate. Moreover,
our simulation results of several other MTL methods, including L{21}, Lasso, and SRMTL (these models can be found and solved by \cite{zhou2011malsar}) also demonstrate the generality of this new insight in MTL.
As shown in \Cref{theory PCR}, the PCR factor is a function of data size, the similarity of the tasks, and regularization parameters, meaning that MTL can achieve a better performance when there exists a good balance of these quantities. Particularly, the improvement of PCR factors in MTL is more significant when the data size is small. 
To conclude, PCR factor is more suitable and accurate to depict the essential advantages of MTL when $T$ is fixed. }

{

{The rest of this} paper is organized as follows. 
Section \ref{section 2} introduced the basic notations of MTL and an extension of M-SVM.
Section \ref{section 3} provided an asymptotic analysis of M-SVM in the limit of large sample size, based on which we raise the main conjecture in this paper.
In Section \ref{section 4}, both the theoretical and experimental analysis were provided to verify the correctness of our conclusions. Section \ref{section 5} presented several discussions, and Section \ref{section 6} concluded this paper. To improve the readability of this paper, we put all proofs in the appendix.

\section{ METHOD}\label{section 2}

\subsection{Problem settings and notations}
We consider the following MTL setting.
Assume we are facing $T$ learning tasks and samples for each task comes out randomly with probability $p(t)$ at each time, for $t=1,\cdots,T$.
After a short time, we collect $N$ data points {totally} from $T$ tasks and suppose that all these samples belonging to the same space $X\times Y$, where $X \subset \mathbb { R } ^ { d }$ and $Y=\left\{-1,1\right\}$. Specifically, 
for each task $t$, there are $m_{t}$ samples generated from the distribution $P _ { t }(\mathcal{X},\mathcal{Y})$, for $t=1,\cdots,T$. That is,
\begin{equation}
\begin{aligned}
&\left\{ \left\{ \left( \mathbf { x } _ { 11} ,y _ { 11} \right) ,\cdots ,\left( \mathbf { x } _ { m_{1} 1} ,y _ { m_{1} 1} \right) \right\}, \cdots, \right. \\
&\left. \left\{ \left( \mathbf { x } _ { 1T} ,y _ { 1T} \right) ,\cdots ,\left( \mathbf { x } _ { m_{T}T} ,y _ { m_{T}T} \right) \right\} \right\}.
\end{aligned}
\end{equation}
For better adapting to the real scenarios, we assume $p(t)$ is different from each task, while their {distribution functions $P _ { t }(\mathcal{X},\mathcal{Y})$ are closely related}. 
The ideal goal of MTL is to learn $T$ functions simultaneously $f _ { 1} ,\cdots ,f _ { T }$ such that $\text{sgn}(f _ { t }) \left( \mathbf { x } _ { i t } \right)= y _ { i t }$, for $i=1,\cdots,m_t$, $t=1,\cdots, T$.

Based on the above settings, we define the average misclassification error for these $T$ classifiers to be the weighted sum of the corresponding misclassification errors \cite{wu2006analysis}: $\mathcal{U}(\text{sgn}(f _ { 1}),\cdots ,\text{sgn}(f _ { T }))=\sum _ { t= 1} ^ { T } p(t)\mathcal{R}_{t} (\text{sgn}(f_{t}))$, where $\mathcal{R}_{t}(f_t)=\mathbb{ E }_{\mathcal { X },\mathcal{Y}}\mathbf { 1} ( \text{sgn}(f_{t}) ( \mathcal { X } ) \neq \mathcal{Y})$, for $t=1,\cdots, T$, and $\mathbf{1}(A)$ is an indicator function with its value being one if the event $A$ is true, and zero if it is not. We see that the average misclassification error actually measures the risk of applying $f _ { 1},\cdots ,f _ { T }$ to make predictions.

We define the minimizer of the misclassification error for each task as $f^*_{t}=\text{arg}\inf_{f_{t}}\mathcal{R}_{t} (f_{t})$, where the infimum is over all measurable functions. Based on \cite{devroye2013probabilistic}, this minimizer has the expression as $f^*_{t}(x)=\text{sgn}(\eta_{t}(x)-\frac { 1} { 2})$ and is called Bayes rule, where $\eta_{t}(x)=P_{t}(\mathcal{Y}=1|\mathcal{X}=x)$, for $t=1,\cdots,T$.

We now define the average expected error of $T$ function $f _ { 1} ,\cdots ,f _ { T }$ with respect to a loss function $\ell : \mathbb { R } \rightarrow [ 0,\infty )$ as $\mathcal { E } ( f _ { 1},\cdots ,f _ { T } ) = \sum _ { t= 1} ^ { T } p(t)\mathbb { E }_{(\mathcal{X},\mathcal{Y})\sim P_{t} }[ \ell (\mathcal{Y}f_{t}( \mathcal{X} ) ) ]$, then the corresponding average empirical error is $ \mathcal { E }_{z}( f _ { 1},\cdots ,f _ { T } )=\sum _ { t= 1} ^ { T }\sum _ { i= 1} ^ { m_{t} }\frac { 1} { N }\ell (y_{it} f_{t}(\mathbf{x}_{it} ) )$.

{This paper mainly concerns with the following denotations. For each task $t$, we define the \textbf{excess misclassification error} of the classifier $\text{sgn}(f_t)$ as $\mathcal{R}(\text{sgn}(f _ { t}))-\mathcal{R}(f^*_ {t} )$. For all the $T$ tasks, the \textbf{average excess misclassification error} of the classifiers $\text{sgn}(f_t)$ ($t=1,\cdots,T$) is defined as $\mathcal{U}(\text{sgn}(f _ { 1}),\cdots ,\text{sgn}(f _ { T }))-\mathcal{U}(f^*_ { 1} ,\cdots ,f^*_ { T })$, and the \textbf{average excess expected error} of the functions $f_t$ ($t=1,\cdots,T$) is defined as $\mathcal { E } ( f _ { 1},\cdots ,f_ { T } )-\mathcal { E } ( f^*_ { 1} ,\cdots ,f^*_ { T } )$.}

\subsection{M-SVM method}
This section extends M-SVM originating from \cite{evgeniou2004regularized} to a more general case, where we assume that the functions $ f_{ 1},\cdots ,f _ { T }$ are nonlinear and the probability of each sample {coming out} from task $t$ is $p(t)$ instead of $ \frac { 1} { T }$. Notice that the exact $p(t)$ is generally unknown, {we use} the sampling frequency $\frac{m_t}{N}$ to approximate $p(t)$ for each task in the problem and obtain a extend version of M-SVM as follows,
\begin{align}\label{svm1}
& \min\limits_ { \tiny \begin{array}{c}
	f_0,f_1,\dots f_T \in \mathcal { H } _{K} \\ \xi_{it}\in \mathbb{R}
	\end{array} }&
	\begin{aligned}
&\bigg\{\sum _ { t = 1} ^ { T } \sum _ { i = 1} ^ { m_{t} }\xi _ { i t } + \lambda _ { 1} \sum _ { t = 1} ^ { T }\frac{m_t}{N} \| g_ { t } \|_{K}^ { 2}  \\
&+ \lambda _ { 2}\| f_ {0} \|_{K} ^ { 2} \bigg\}
\end{aligned}
\\
&\text{s.t.}& 
\begin{aligned}
&y_{it}\cdot f_t(\textbf{x}_{it})\geq
1- \xi _ { i t} , \quad \xi _ { it} \geq 0, \\
&i= 1,\dots ,m_t; \quad t=1,\dots,T.
\end{aligned}
\notag
\end{align}
where $\lambda_{1}$ and $\lambda_{ 2}$ are two positive regularization parameters, $K$ is a universal kernel (e.g, the Gaussian kernel), $\mathcal { H } _{K}$ is the Reproducing Kernel Hilbert Space (RKHS) \cite{wu2006analysis}, 
 $\|\cdot\|_{K}$ is the norm function in this Hilbert space, and $g_{t}:=f_t-f_0$. {Since \eqref{svm1} contains regularization terms $\|g_t\|^2_K$ and $\|f_0\|^2_K$ that models task relations, we refer this as a regularized MTL model. }
{In Eq. \eqref{svm1}, $f_{0}$ represents the commonness of those classifiers, while $g_{t}$, $t=1,\cdots, T$, represent their individualities. }
The extended model of \eqref{svm1} will be reduced to the original one \cite{evgeniou2004regularized} when $m_{t}$ ($t=1,\cdots, T$) are identical. 
Note that considering the randomness in the sampling frequency for each task makes the proposed method better adaptable to real applications, {as the model pays more attention to the higher frequency tasks}.
Hence, the M-SVM Eq. {\eqref{svm1}} well extends the original MTL framework. By the positive definiteness and convexity of the loss function and the norm function, there exists the optimal solution of \eqref{svm1}.

Following the similar calculation scheme of \cite{evgeniou2004regularized},  {and denoting ${f}^{z}_\text{t} (t=1,\cdots,T)$, ${f}^{z}_{0}$} to be the optimal solution of Eq. \eqref{svm1}, we can get a relation of these quantities in \Cref{lemma1}, and then reformulate Eq. \eqref{svm1} as Eq. \eqref{svm2} in \Cref{theorem1}.
\begin{lemma}[{Adapted from \mbox{\cite[Lemma 2.1]{evgeniou2004regularized}}}]\label{lemma1}
	The optimal solution to the Eq. \eqref{svm1} satisfies the equation
$	{f}^{z}_ { 0}  = \frac { \lambda _ { 1} } { \lambda _ { 2} + \lambda _ { 1} } \sum _ { t = 1} ^ { T } \frac { m_t} { N} { f}^{z} _ { t }$.
\end{lemma}

\begin{theorem}[{Adapted from \mbox{\cite[Lemma 2.2]{evgeniou2004regularized}}}]\label{theorem1}
	The multi-task problem Eq. \eqref{svm1} is equivalent to the optimization problem below:	
	\begin{align}\label{svm2}
	& \min\limits_ { \tiny \begin{array}{c}
		f_1,\dots f_T \in \mathcal { H } _{K} \\ \xi_{it}\in \mathbb{R}
		\end{array} }&
		\begin{aligned}
	&\bigg\{\sum _ { t = 1} ^ { T } \sum _ { i = 1} ^ { m_{t} }\xi _ { i t } + \rho _ { 1} \sum _ { t = 1} ^ { T }\frac{m_t}{N} \| f_ { t } \| _{K}^ { 2}\\
	&+ \rho _ { 2}\sum _ { t = 1} ^ { T } \frac{m_t}{N}\left\| f _ { t } -\sum _ { s = 1} ^ { T }\frac{m_s}{N} f _ { s } \right\|_{K}^ { 2}\bigg\}
		\end{aligned}
	\\
	&\text{s.t.}&
	\begin{aligned}
	&y_{it}\cdot f_t(\textbf{x}_{it})\geq
	1- \xi _ { i t} , \quad
	\xi _ { it} \geq 0, \\
	&i= 1,\dots ,m_t; \quad  t=1,\dots,T. 
	\end{aligned}
	\notag
	\end{align}
	where $\rho _ { 1} =\frac { \lambda _ { 1} \lambda _ { 2} } { \lambda _ { 1} + \lambda _ { 2} }$ and $\rho _ { 2} = \frac { \lambda _ { 1} ^ { 2} } { \lambda _ { 1} + \lambda _ { 2} }$.
\end{theorem}


{Theorem} \ref{theorem1} indicates that M-SVM works by achieving a good trade-off between independent learning (i.e., the second term in \eqref{svm2}) and aggregate learning that treats different tasks as the one (i.e., the last term in \eqref{svm2}). 

In the next section, we will show an asymptotic property of M-SVM in Eq. \eqref{svm2} that it is Bayes risk consistent in the limit of large sample size.

\section{Asymptotic Performance of M-SVM}\label{section 3}
In this section, we first present the average excess misclassification error bound for M-SVM \eqref{svm2}. Then, based on the analysis of this bound, we raise the main conjecture that the superiority of M-SVM is the improvement of the PCR factor rather than the convergence rate.

	\subsection{Error analysis of M-SVM}
	
This subsection derives an upper bound of the average excess misclassification error for M-SVM \eqref{svm2}: $\mathcal{U}(\text{sgn}(f^z _ { 1}),\cdots ,\text{sgn}(f^z _ { T }))-\mathcal{U}(f^*_ { 1} ,\cdots ,f^*_ { T })$, by extending the analysis techniques for the single-task setting \cite{steinwart2008support,chen2004support,wu2006analysis,fan2017learning,wu2007multi}. {The definition of this error is given in Section II-A.}
	
	
	First, we show the minimizers of the average expected error $\mathcal { E } ( f _ { 1},\cdots ,f _ { T } )$ {to be the Bayes rules}.
	
	\begin{theorem}\label{theorem2}
		The minimizers of $\mathcal { E } ( f _ { 1},\cdots ,f _ { T } )$ {over all measurable functions are the Bayes rules} $f^*_{t}$, where $t=1, \cdots, T$.
	\end{theorem}
	\begin{proof}
		Note that the minimizer of {each} $\mathbb { E }_{(\mathcal{X},\mathcal{Y})\sim P_{t} }[ \ell (\mathcal{Y}f_{t}( \mathcal{X} ) ) ]$ is the Bayes rule $f^*_t$ over all measurable functions \cite[\textit{lemma 3.1}]{lin2002support}.
		{By the definition, the Bayes rules also minimize $\mathcal { E } ( f _ { 1},\cdots ,f _ { T } )$. }
	\end{proof}
	Theorem \ref{theorem2} implies one advantage of the MTL framework of \mbox{\eqref{svm2}} that it leverages the related tasks by introducing a regularization term in the objective function without imposing the additional restrictions on the functional space \cite{ando2005framework,maurer2016benefit}, and, therefore, the best rule (Bayes rule) is not precluded.
	

	{Similar to the single-task case, there exists a bridge inequality which connects the average excess misclassification error with the average excess expected error.}
	\begin{theorem}\label{theorem3}
		For any $h_{t}:X \rightarrow \mathbb { R }, t=1, \cdots, T$, there exists a bridge inequality between the average misclassification error $ \mathcal{U}(h _ { 1} ,\cdots ,h _ { T })$ and the average expected error $\mathcal { E } ( h _ { 1} ,\cdots ,h_ { T } )$:
		\begin{eqnarray}
		\begin{aligned}
		&\mathcal{U}(\text{sgn}(h _ { 1}) ,\cdots ,\text{sgn}(h _ { T }))-\mathcal{U}(f^*_ { 1} ,\cdots ,f^*_ { T }) \\
		&\leq \mathcal { E } ( h _ { 1},\cdots ,h_ { T } )-\mathcal { E } ( f^*_ { 1} ,\cdots ,f^*_ { T } ).
		\end{aligned}
		\end{eqnarray}
	\end{theorem}	
	\begin{proof}
		Notice that for each $h_t: X \rightarrow \mathbb { R }$ the inequality
		$\mathcal { R }_t (\text{sgn}(h_t )) - \mathcal { R }_t \left( f^*_t \right) \leq \mathcal { E}_t ( h_t ) - \mathcal { E }_t \left( f^*_t \right)$ has established in \cite{zhang2004statistical}, where $\mathcal { E }_t \left( h_t \right)=\mathbb { E }_{(\mathcal{X},\mathcal{Y})\sim P_{t} }[ \ell (\mathcal{Y}h_{t}( \mathcal{X} ) ) ]$.
		Then by the definitions of  $ \mathcal{U}(h _ { 1},\cdots ,h _ { T })$ and $\mathcal { E } ( h _ { 1},\cdots ,h_ { T } )$, we can get the desired result directly.
	\end{proof}
Let $(f^{\mathcal { H }}_{1}, \cdots, f^{\mathcal { H }}_{T})$ be the minimizer of the average expected error over the RKHS space $\prod_{i=1}^{T} \mathcal { H } _{K} $ , i.e.,
	\begin{equation*}
	\begin{aligned}
	&(f^{\mathcal { H }}_{1}, \cdots, f^{\mathcal { H }}_{T}) = \argmin_{{ h _ { 1}, \cdots ,h _ { T }\in\mathcal { H } _{K} }} \bigg\{\mathcal { E }(h _ { 1},\cdots ,h _ { T })\\
	&+\frac{\rho _ { 1}}{N} \sum _ { t = 1} ^ { T }p(t) \| h_ { t } \|_{K} ^ { 2}
	+ \frac{\rho _ { 2}}{N}\sum _ { t = 1} ^ { T }p(t) \left\| h_ { t } -\sum _ { s = 1} ^ { T }p(s) h _ { s } \right\|_{K}^ { 2}\bigg\},
	\end{aligned}
	\end{equation*}
	{where the existence of $(f^{\mathcal { H }}_{1}, \cdots, f^{\mathcal { H }}_{T})$ is guaranteed by the positive definiteness and convexity of the loss function and the norm function.}
	{Then we have the following lemma, which gives an upper bound of the average excess expected error.}
	\begin{lemma}\label{lemma2}
		The following inequality holds:
		\begin{equation*}
		\begin{aligned}
		&\mathcal{E}(f ^z_ { 1}  ,\cdots ,f^z _ { T })-\mathcal{E}(f^*_ { 1}  ,\cdots ,f^*_ { T }) \\
		&\leq \mathcal { S } ( N,T) +\mathcal { D } (N, T,\rho_{ 1},\rho_{ 2})+\mathcal{F}(N,T,\rho_1, \rho_2)
		\end{aligned}
		\end{equation*}	
		where
		\begin{equation*}
		\begin{aligned}
		\mathcal { S } (N,T)&:=\left\{ \mathcal { E } \left(f ^z_ { 1}, \cdots ,f^z _ { T } \right) - \mathcal { E } _ { z } \left( f ^z_ { 1} ,\cdots ,f^z _ { T }\right) \right\}\\
		&+ \left\{ \mathcal { E } _ { z } \left( f^{\mathcal { H }}_{1}, \cdots, f^{\mathcal { H }}_{T} \right) - \mathcal { E } \left( f^{\mathcal { H }}_{1}, \cdots, f^{\mathcal { H }}_{T} \right) \right\},
		\end{aligned}
		\end{equation*}
		\begin{equation*}
		\begin{aligned}
		\mathcal { D } (N,T, \rho_1, \rho_2 )&:=\inf_{ h _ { 1}, \cdots ,h _ { T }\in\mathcal { H } _{K} } \Bigg\{ [\mathcal { E }(h _ { 1}  ,\cdots ,h _ { T })\\
		&-\mathcal{E}(f^*_ { 1},\cdots ,f^*_ { T })] \notag
		+\frac{\rho _ { 1}}{N} \sum _ { t = 1} ^ { T }p(t)\| h_ { t } \|_K ^ { 2} \\
		&+ \frac{\rho _ { 2}}{N}\sum _ { t = 1} ^ { T } p(t)\left\| h_ { t } -\sum _ { s = 1} ^ { T }p(s) h _ { s } \right\|_K^ { 2}\Bigg\}, 
		\end{aligned}
		\notag
		\end{equation*}
		and
		\begin{align*}
		&\mathcal{F}(N,T,\rho_1, \rho_2) = -\frac{\rho_{ 1}}{N}\sum_{ t = 1}^{T} \left(p(t)-\frac{m_{t}}{N}\right)\| f^{\mathcal{H}}_ { t } \|_{K} ^ { 2}  \notag\\
		& +\frac{\rho_2}{N}\sum_{t=1}^{T}\left(\frac{m_{t}}{N}-p(t)\right)\left\| f^{\mathcal{H}}_ { t } -\sum _ { s = 1} ^ { T }\frac{m_s}{N} f^{\mathcal{H}}_{s} \right\|^2_{K}\notag\\
		& +\frac{2\rho _ {2}}{N}\sum _ { t = 1} ^ { T }p(t)\left\| f^{\mathcal{H}}_ { t } -\sum _ { s = 1} ^ { T }\frac{m_s}{N} f^{\mathcal{H}}_{s} \right\|_{K}\left\| \sum _ { t = 1} ^ { T } \left( p(t) - \frac{m_t}{N}\right) f^{\mathcal{H}}_{t} \right\|_{K}\notag
		\end{align*}
	\end{lemma}	
	The proof of \textit{Lemma \ref{lemma2}} can be found in the appendix.
	
	Similar to \cite{wu2006analysis}, we refer to the first term $\mathcal { S } (N, T)$ as the average sample error and the second term $\mathcal { D } ( N, T,\rho_{ 1},\rho_{ 2})$ as the average regularization error. It is worth to notice that we have an additional term $\mathcal{F}(N, T, \rho_1, \rho_2)$ in \textit{Lemma \ref{lemma2}}, compared to literature \mbox{ \cite{wu2006analysis}}}. Here we refer to it as the frequency error, because one of its dominant term, $p(t)-\frac{m_{t}}{N}$, measures the sampling frequency error of each task.
	{The frequency error vanishes in the single-task learning cases where the sampling frequency is deterministic, i.e. $m_{1}=N$.}
	{Therefore, the frequency error is a new knowledge associated with the MTL problems only.}
	
	
	To bound these three errors, we have the following lemmas.
	\begin{lemma}\label{lemma3}
		An upper bound of the average sample error can be given by
		$\mathcal{S}(N,T) = \mathcal{O} \left(N^{-1/4+\epsilon}\right)$ almost surely,
		where $\epsilon$ is any positive constant.
	\end{lemma}	
	
	\begin{lemma}\label{lemma4}
	   Assume that the Bayes rules for each task $f^*_t$ ($t=1,\dots,T$) are restrictions of some functions $\tilde f^*_t$ lying in the Sobolev space $H^s(\mathcal R^d)$.
		Then, an upper bound of the average regularization error can be given by
		$\mathcal{D}( N,T,\rho_{ 1},\rho_{ 2}) = \mathcal{O}\left(( \log N ) ^ { - s/4 }\right)$ almost surely.
	\end{lemma}	
	\begin{lemma}\label{lemma5}
		An upper bound of the frequency error can be given by
		$\mathcal{F}(N,T,\rho_{ 1},\rho_{ 2}) = \mathcal{O} \left(N^{-3/2+\epsilon}\right)$ almost surely, where $\epsilon$ is any positive constant.
	\end{lemma}
	The proof of \textit{Lemmas \ref{lemma3}, \ref{lemma4} and \ref{lemma5}} can be found in the appendix.
	
	Combining the results of \textit{Theorem \ref{theorem3}}, \textit{Lemmas \ref{lemma2}, \ref{lemma3}, \ref{lemma4} and \ref{lemma5}}, we can provide the main result of this section in the theorem below.
	\begin{theorem}\label{theorem4}
		{Assume that $f^*_t$ ($t=1,\dots,T$) are the restrictions of some functions $\tilde f^*_t$ lying in the Sobolev space $H^s(\mathcal R^d)$.}
		Then, an upper bound of excess average misclassification error can be obtained almost surely by
		\begin{equation}
		\mathcal{U}(\text{sgn}(f ^z_ { 1})  ,\cdots ,\text{sgn}(f^z _ { T }))-\mathcal{U}(f^*_ { 1}  ,\cdots ,f^*_ { T }) = \mathcal{O}\left(( \log N ) ^ { - s/4 } \right).
		\end{equation}	
	\end{theorem}	
	\begin{proof}
		{It follows from \textit{Theorem \ref{theorem3}}, \textit{Lemmas \ref{lemma2}, \ref{lemma3}, \ref{lemma4} and \ref{lemma5}}.}
	\end{proof}

	\begin{cor}
		Almost surely, the misclassification error $\mathcal{R}_t(f^z_{t})$ for each task $t$, ($t=1,\cdots,T$), can be arbitrarily close to the corresponding Bayes error $\mathcal{ R }_t(f_t^*)$, as long as $N$ is sufficiently large.
	\end{cor}	

	Theorem \ref{theorem4} indicates that the average misclassification error $\mathcal{U}(\text{sgn}(f ^z_ { 1})  ,\cdots ,\text{sgn}(f^z _ { T }))$ {can be arbitrarily} close to the average Bayes error $\mathcal{U}(f^*_ { 1}  ,\cdots ,f^*_ { T })$, as long as $N$ is sufficiently large.
	Moreover, the Corollary above shows that the learned classifier $\text{sgn}(f^z_{t})$ for each task $t$ converges to the corresponding Bayes rule $f^*_{t}$ , when $N$ increases to infinity. Therefore, given a large data size, the M-SVM \eqref{svm2} induces a reliable classification rule for each task with a small  misclassification error, which gives us great confidence in this method.
	
	\subsection{Important remarks} \label{section further remarks}
	
	\subsubsection{Task-interaction vanishes in the limit of large sample size}{\label{subsubsection interactions}}
	Previous literature has shown the Bayes risk consistency of the single-task SVM \cite{wu2006analysis}. Based on this fact, our main result in Theorem \ref{theorem4} implies that applying the M-SVM \eqref{svm2} amounts to learning $T$ tasks independently, when the number of samples is sufficiently large.  
	With this regard, in the MTL framework, the interaction between different tasks vanishes as the sample size grows greatly. 
	{Intuitively, we can observe this point through the objection function of M-SVM \eqref{svm2}, that the empirical loss dominates the regularization terms (i.e., the task-interaction terms) when the data size is large; therefore, the influence of the task interaction vanishes in the limit of large sample size.} This result implies that when the data size is large enough, M-SVM \eqref{svm2} still makes the reliable decisions even if the tasks considered are significantly deviated from each other. This is quite different from the one in finite sample cases, where the task interaction cannot be negligible and plays a key role in improving the performance of each task. Moreover, {as shown in \cref{simulation studies} and \Cref{experiments on blast furnace dataset}, the performance curves of M-SVM and its single-task counterpart on the simulated and real datasets are initially separable and finally overlap, which supports this point again.}

	\subsubsection{Regularization terms are necessary}\label{section regularization term}
	Note that the average regularization error $D(N,T,\rho_{ 1},\rho_{ 2})$ disappears if $\rho_{ 1}=\rho_{ 2}=0$.
	In this case, by the results of \text{Lemma \ref{lemma2}, \ref{lemma3}, \ref{lemma4}, and \ref{lemma5}}, it seems that a faster convergent rate $\mathcal{O}(N^{-\frac{1}{4}+\epsilon})$ determined by average sample error can be achieved.
	However, it is not the case.
	When $\rho_{ 1}=\rho_{ 2}=0$, we will lose the upper bound estimation of $\|f_{t}^{z}\|_{K}$ (c.f. \textit{lemma \ref{lemma 6}} in the appendix) and, therefore, fail to analyze the average sample error.
	To the best of our knowledge, there has been no literature succeed in estimating the sample error without the boundedness of $\|f_{t}^{z}\|_{K}$.
	In this regard, regularization terms are necessary. 
	In general, the average excess misclassification error bounds are determined by the complexity of the RKHS \cite{zhou2002covering,zhou2013density,wu2006analysis,ying2007learnability}.
	{Our experimental results in \Cref{simulation studies} and  \Cref{experiments on blast furnace dataset} also show that, without the regularization term, the M-SVM classifier $\text{sgn}(f^z_t)$ for each task cannot converge to the corresponding Bayes rule $f^*_t$. Therefore, it is important to include the regularization term in MTL. }
	
	\subsubsection{What benefits does the M-SVM bring?} \label{section benefit of multi}
	{
	Remark \ref{subsubsection interactions}	has shown the equivalence of the M-SVM and its single-task counterpart in terms of their limit classifiers' performance.
	Meanwhile, both methods' convergence rates have the same upper bound $\left(\mathcal{O}( \log N ) ^ { - s/4 } \right)$ (see \textit{Theorem \ref{theorem4}} and \cite{wu2006analysis}) determined by the complexity of the RKHS.
	Though obtaining the optimal convergence rate requires more elaborate analysis of the RKHS \cite{wu2006analysis}, we conjecture that both methods' convergence rate are the same when the number of tasks is fixed, as the more elaborate analysis should be applicable for both cases. 
	{Consequently, we conjecture that the benefit of M-SVM is the improvement of its pre-convergence-rate (PCR) factor for each task, defined by $C_M(m_t)=\frac{ \mathcal{R}(\text{sgn}(f ^z_ { t}) )-\mathcal{R}(f^*_ {t} ) } {F(m_t)}$, where $1/F(m_t)$ is the M-SVM's exact convergence rate with $F(m_t)\rightarrow\infty ~(m_t\rightarrow\infty)$. Similarly, for each task, we define the PCR factor of SVM and denote it by $C_S(m_t)$.
	The following section gives the theoretical (in a one-dimensional least-square classification setting) and experimental justifications of this important claim. Our results demonstrate that, in the similar (non-similar)-task setting, there holds $C_M(m_t)<C_S(m_t)$ ($C_M(m_t)>C_S(m_t)$) when $m_t$ is relatively small, and $C_M(m_t)-C_S(m_t)=0$ as $m_t\rightarrow\infty$.}

    \section{Case studies}\label{section 4}
   {This section contains three case studies to verify our claims in \Cref{section further remarks}. Specifically,  we first provide a theoretical verification of the improvement of the PCR factor for MTL in a one-dimensional least-square classification setting. Then, we provide the empirical results of the MTL models, including M-SVM (with linear and Gaussian kernels), unregularized M-SVM, L21, LASSO, and SRMTL (the later three models can be found and solved by \cite{zhou2011malsar}) on the simulated data. Finally, we conduct experiments on the blast furnace dataset using the M-SVM and its unregularized version with linear kernel. The simulated and experimental results achieve an excellent agreement with the claims in \Cref{section further remarks}. Without loss of generality, we only consider MTL with two tasks and display its performance for the first task (we can obtain a similar result for the second task cause MTL treats both tasks equally).}
    
    
    \subsection{A theoretical justification of PCR factor}\label{theory PCR}
{This subsection provides a theoretical guarantee of the PCR factor's improvement in MTL on a one-dimensional classification problem. Explicitly, we analyze the ability of MTL originating from \cite{evgeniou2004regularized} based on the least-square classifiers due to its availability of the analytic solution. }

{We assume that data are sampled from a two-class Gaussian model with ${x}_1|(y_1=1)\sim \mathcal{N}(\mu^+_1,\sigma_1^2)$, ${x}_1|(y_1=-1)\sim \mathcal{N}(\mu_1^-,\sigma_1^2)$ for the first task, and ${{x}}_2|({y}_2=1)\sim \mathcal{N}({\mu}_2^+,{\sigma}^2_2)$, ${{x}_2}|({y}_2=-1)\sim \mathcal{N}({\mu}^-_2,{\sigma}^2_2)$ for the second task. For each task, we assume there exists no sampling bias of the positive and negative classes. We also assume $\mu_1^+>\mu_1^-$ and denote $\delta^+={\mu}_2^+-{\mu}_1^+$, $\delta^-={\mu}_2^--{\mu}_1^-$. Based on the problem setting above, the decision boundaries for these two tasks can be expressed as $w_1x+b_1=0$ and $w_2x+b_2=0$, respectively. Without loss of generality, we can further suppose $w_1=w_2$, where $w_1$ and $w_2$ are two binary variables and only take two values $-1$ or $1$ (will be shown later).}

{Combining the MTL \cite{evgeniou2004regularized} with the least-square classifiers can naturally produce the multi-task least-square classification model (referred to as ``M-LSC'') in one-dimensional setting as follows,
}

\begin{align}
& \min\limits_ { \tiny \begin{array}{c}
	b_0,b_1,b_2, w_1,w_2
	\end{array} }&
\begin{aligned}
&\bigg\{\sum _ { t = 1} ^ { 2 }\sum _ { i = 1} ^ { m_1 }(y_t^i-w_tx_t^i-b_t)^2+ \frac{\lambda _ { 1}}{2} \sum _ { t = 1} ^ { 2 } v_ { t }^2 \bigg\},\notag
\end{aligned}
\end{align}
{where $b_1=b_0+v_1$ and $b_2=b_0+v_2$ are two parameters of the decision boundaries for the M-LSC classifiers. }

{By the method of Lagrange multipliers and the assumption that $w_1=w_2$, we can easily obtain the optimal solution of the M-LSC problem for the first task as:
$b_1^*=\frac{(G-1)}{m_1} \sum_{i=1}^{m_1} w_1^*x_1^i+\frac{G}{m_2}\sum_{i=1}^{m_2}w_2^*x_2^i$,}
{where $w_1^*=w_2^*=\text{sgn}\left(\sum_{t,i} (x^i_t)^{+} -\sum_{t,i} (x^i_t)^{-}\right)$, and $G=\frac{\lambda_1m_1}{4m_1m_2+\lambda_1(m_1+m_2)}$. Then, by the definition of Bayes classifier and the problem setting above, we can also get the Bayes classifier for the first task as $f_1^*=\text{sgn}\left(x-\frac{\mu_1^++\mu_1^-}{2}\right)$. }

{Based on the results above, we can now estimate the excess misclassification error for the M-LSC estimate with respect to the first task $f_1=\text{sgn}(w_1^*(x-x^0))$, where $x^0=\frac{-b_1^*}{w_1^*}$. As will be shown below, we only need to estimate the MSE of $x^0$.}

{Note that, for the fixed $x^0$ and $w_1^*$, there holds $\mathcal{R}(\text{sgn}(f _ { 1}))=\frac{1}{2}\left(1-w_1^*[\phi(x^0-\mu_1^-)-\phi(x^0-\mu_1^+)]\right)$, where $\phi(\cdot)$ is the CDF of the standard Gaussian distribution. Then, we have $\mathcal{R}(\text{sgn}(f^* _ { 1}))=\frac{1}{2}\left(1-[\phi(\frac{\mu_1^+-\mu_1^-}{2})-\phi(\frac{\mu_1^--\mu_1^+}{2})]\right)$. By Taylor's theorem, and denote $h=x^0-\frac{\mu_1^-+\mu_1^+}{2}$, we can check that for the $x^0$ and $w_1^*$ learned by M-LSC, there holds}
\begin{align}\label{reformulate loss}
\mathcal{R}(\text{sgn}(f _ { 1}))-\mathcal{R}(\text{sgn}(f^* _ { 1}))= D\cdot \mathbb{ E }(h)^2+\mathcal{O}(\mathbb{ E }(h)^4),
\end{align}
{where $D=\left[\phi^{''}(\frac{\mu_1^--\mu_1^+}{2})-\phi^{''}(\frac{\mu_1^+-\mu_1^-}{2})\right]/4$ is a constant.}

{By some standard calculations, we can obtain that } 
\begin{align}\label{pcr1}
\mathbb{ E }(h)^2= \frac{(1-G)^2\sigma_1^2}{m_1}+\frac{G^2\delta_1^2}{4}+\frac{G^2\sigma_2^2}{m_2},
\end{align}
{where $\delta=\delta^++\delta^-$ is the task dissimilarity. Notice that we can degenerate M-LSC results to its single-task counterpart (S-LSC) by making $\lambda_1=0$. Then, we see that the real convergence rates of M-LSC and S-LSC are both $\mathcal{O}(\frac{1}{m})$ by assuming $m_1=m_2=m$. According to the PCR factor's definition in \Cref{section benefit of multi}, the PCR factor of M-LSC is $(1-G)^2\sigma_1^2+G^2\sigma_2^2+G^2\delta_1^2/4$, and $\sigma_1^2$ for the S-LSC case. We describe the PCR factor as a function of data size in \Cref{pcr-figure}.} 

\begin{figure}[!htbp]
	\centering
	\includegraphics[width=0.45\textwidth]{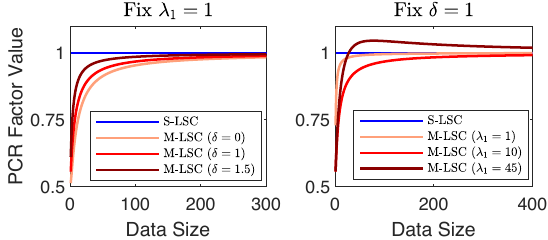}
	\captionsetup{singlelinecheck = false, justification=justified}
	\caption{The change of the PCR factor with the data size $m$, regularization parameter $\lambda_1$, and task dissimilarity $\delta$. Here, $\sigma_1=\sigma_2=1$, and $m_1=m_2=m$.}
	\label{pcr-figure}
\end{figure}

{Our theoretical analysis above demonstrates that PCR factors are different in S-LSC and M-LSC cases. As shown in \Cref{pcr-figure}, M-LSC has a smaller PCR factor value when the data size is relatively small, whereas its PCR factor value converges to that of the S-LSC as $m\rightarrow\infty$. This shows that the priority of M-LSC lies in reducing the PCR factor when the data size is relatively small. Furthermore, the PCR factor also depends on the task dissimilarity $\delta$ (see the left panel of \Cref{pcr-figure}); the less $\delta$ is, the less the PCR factor will be.
This occurs because more similar tasks contain more common knowledge that can be shared between tasks. 
Moreover, the PCR factor is related to the regularization parameter $\lambda_1$ (see the right panel of \Cref{pcr-figure});
particularly, M-LSC can be worse than S-LSC when $\lambda_1$ is not properly chosen.
Therefore, the capacity of MTL also relies on the choice of tuning parameters. }

{We theoretically verified that in this case study, the benefit of MTL lies in the improvement of the PCR factor rather than the convergence rate. 
In what follows, our experiments on the simulated and real data sets demonstrate that this conclusion universally holds in general MTL frameworks.  
}

\subsection{Simulation studies}\label{simulation studies}


\begin{figure*}[!htbp]
    \centering
    \includegraphics[width=0.9\textwidth]{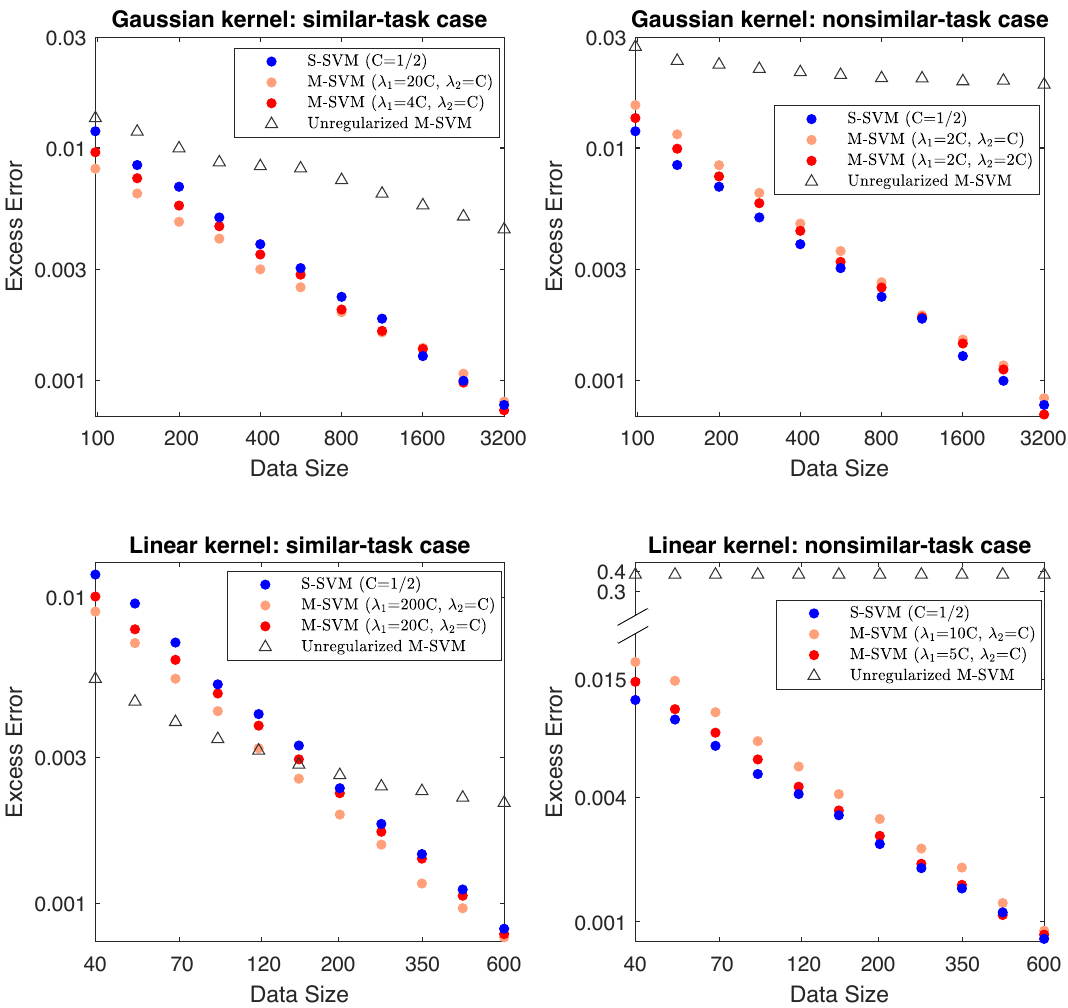}
    \captionsetup{singlelinecheck = false, justification=justified}
    \caption{
    This figure shows the log-log plot of the excess misclassification error (excess error) to the training data size on the simulated dataset using the S-SVM (with the regularization parameter $C=1/2$), M-SVM, and unregularized M-SVM in the presence of the randomness in the sampling frequency. 
    For rigorous comparisons, we fixed one of the regularization parameters ($\lambda_1=2C$ or $\lambda_2=C$) for the M-SVM in each case to make it comparable with the S-SVM. The plots of these two models are parallel initially, then get closer gradually and overlap finally, indicating that they have the same convergence rate but different PCR factors. Without the regularization terms, the unregularized M-SVM classifier cannot converge to the corresponding Bayes rule. }
    \label{fig1}
 \end{figure*}

{In this subsection, we first design two simulation studies of M-SVM (one with Gaussian (set its width as $1$) and one with linear kernels) to support our claims in \Cref{section further remarks} under two different task dissimilarity settings. 
To mimic the scenarios in real problems, the data are generated from multi-variate distributions and linearly non-separable, and the uncertainty of the data frequency is also considered.
Then, we show that the simulation results of several other MTL methods including L21, LASSO, and SRMTL (these models can be found and solved by \cite{zhou2011malsar}) are also consistent with the claims.}

{For the similar-task setting, we generate data for the first task from a two-class Gaussian model with $\textbf{x}_1|(y_1=1)\sim \mathcal{N}([\begin{smallmatrix} 2\\3 \end{smallmatrix}],[\begin{smallmatrix} 1&0\\0&2 \end{smallmatrix}])$, $\textbf{x}_1|(y_1=-1)\sim \mathcal{N}([\begin{smallmatrix} 4\\5 \end{smallmatrix}],[\begin{smallmatrix} 1&0\\0&2 \end{smallmatrix}])$. We generate data for the second task from the other two-class Gaussian model with $\textbf{x}_2|(y_2=1)\sim \mathcal{N}([\begin{smallmatrix} 2.2\\3.2 \end{smallmatrix}],[\begin{smallmatrix} 1&0\\0&2 \end{smallmatrix}])$, $\textbf{x}_2|(y_2=-1)\sim \mathcal{N}([\begin{smallmatrix} 4.0\\5.4 \end{smallmatrix}],[\begin{smallmatrix} 1&0\\0&2 \end{smallmatrix}])$. For the nonsimilar-task case, we generate data for the first task the same as the above, while we assume data for the second task are sampled from $\tilde{\textbf{x}}_2|(\tilde{y}_2=1)\sim \mathcal{N}([\begin{smallmatrix} 7\\8 \end{smallmatrix}],[\begin{smallmatrix} 1&0\\0&2 \end{smallmatrix}])$, $\tilde{\textbf{x}}_2|(\tilde{y}_2=-1)\sim \mathcal{N}([\begin{smallmatrix} 5\\12 \end{smallmatrix}],[\begin{smallmatrix} 1&0\\0&2 \end{smallmatrix}])$. We assume the probability of each sample coming out from task $t$ is $0.5$ (that is, $m_t/N\approx 0.5$ when $N$ is large enough), for $t=1,2$.}

{We apply the single-task SVM (S-SVM), M-SVM (with different parameters) and unregularized M-SVM to the classification problems above. The excess error is summarized in \Cref{fig1} where each dot represents the average of $2000$ random experiments based on $10000$ test samples. In all cases, we see that the log-log plots of the S-SVM and M-SVM are initially parallel, then they gradually get closer and finally overlap, indicating that they have the same convergence rate but the different PCR factors. Particularly, for the similar-task case (e.g., the top left panel of \Cref{fig1}), $C_M(m_1)$ ($<C_S(m_1)$) gets smaller and approaches to $C_S(m_1)$ as the growth of $\lambda_1$ with fixed $\lambda_2$. It is because when $\lambda_1$ becomes more significant, M-SVM makes the tasks share more information from each other. When we fix $\lambda_1$ and increase $\lambda_2$ (e.g, the top right panel), $C_M(m_1)$ ($>C_S(m_1)$) becomes smaller and gets closer to $C_S(m_1)$ in the nonsimilar-task case. This occurs because M-SVM tends to learn each task independently when $\lambda_2$ grows greatly.
Moreover, in all cases, the positive or negative improvement of PCR factors in MTL is more significant when the data size is relatively small, and it vanishes when the data size goes to infinity. 	
These observations are consistent with our theoretical result in Theorem \ref{theorem4}, and our claims in \Cref{subsubsection interactions} and \Cref{section benefit of multi}.}

{For the unregularized M-SVM, that is the M-SVM with its regularization parameters being zero, we see that it cannot converge to the Bayes classifier in all scenarios. This is expected because, 
the unregularized M-SVM cannot leverage the dissimilarity of tasks, and it simply treats two different tasks as the one.	This fact supports our discussions in \Cref{section regularization term}.}

\begin{table}[htp]
	\centering
	\caption{M-SVM with Gaussian Kernel}
	\begin{tabular}{|c|c|c|c|c|c|}
		\hline
		Date size & $200$ & $400$ & $800$ & $1600$ & $3200$ \\
		\hline 
		EE-NRF$^{\rm a}$ & $0.0053$ & $0.0032$ & $0.0020$ & $0.0012$ & $0.0008$ \\
		\hline
		PPD-SF$^{\rm b}$(\%)& $10.74\%$ & $5.78\%$ & $0.26\%$ & $10.70\%$ & $4.30\%$ \\
		\hline
	\end{tabular}
\begin{tablenotes}
	\footnotesize
	\item EE-NRF$^{\rm a}$: the excess error of M-SVM without the randomness of the sampling frequency.
	\item PPD-SF$^{\rm b}$: the (absolute) percentage performance difference of the M-SVM with/without the randomness of the sampling frequency. We define it as the excess error difference of these two models as a percentage of the excess error of the M-SVM. The results are computed in the similar-task setting with $\lambda_1=20C, \lambda_2=C (C=1/2)$.
\end{tablenotes}
\label{table1}
\end{table}

\begin{table}[htp]
	\centering
	\caption{M-SVM with Linear Kernel}
	\begin{tabular}{|c|c|c|c|c|c|}
		\hline
		Date size & $70$ & $120$ & $200$ & $350$ & $600$ \\
		\hline 
		EE-NRF$^{\rm a}$ & $0.0061$ & $0.0038$ & $0.0023$ & $0.0014$ & $0.0008$ \\
		\hline
		PPD-SF$^{\rm b}$
		 (\%)& $1.47\%$ & $0.48\%$ & $3.15\%$ & $4.86\%$ & $2.86\%$ \\
		\hline
	\end{tabular}
    \label{table2}
\end{table}

{To quantify the effects caused by the randomness of the sampling frequency in the multi-task problem, we compute the absolute percentage performance difference (in terms of the excess error) for the M-SVM with/without this randomness in \Cref{table1} and \Cref{table2}. We see that, in both cases, it has a small but non-negligible effect on the M-SVM's performance, and it does not change the convergence rate (as the absolute percentage differences are relatively small when the data size is large enough). This also supports our theoretical result in \Cref{theorem4}.}


{We also verify our claim of \Cref{section benefit of multi} for several other MTL frameworks, including L{21}, Lasso, and SRMTL (these models can be found and solved by \cite{zhou2011malsar}). 
 \Cref{fig2} again shows that these MTL methods and their single-task counterparts have the same convergence rate (the log-log plots finally overlap) but the different PCR factors (the log-log plots are initially parallel and gradually get closer). The improvement of PCR factors in MTL is more significant when the data size is relatively small. These observations imply that our claims in \Cref{subsubsection interactions} and \Cref{section benefit of multi} hold in the general MTL frameworks.}

\begin{figure}[!htbp]
	\centering
	\includegraphics[width=0.45\textwidth]{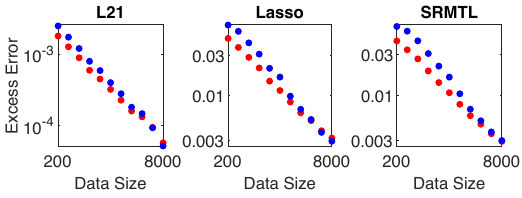}
	\captionsetup{singlelinecheck = false, justification=justified}
	\caption{Log-log plot of the excess error for the MTL methods (red) including L21, Lasso, SRMTL, and their individual counterparts (blue) in the similar-task setting. The regularization parameters for these models are 0.1, 0.4 and 0.4 (the graph regularization parameter is 0.01), respectively.}
	\label{fig2}
\end{figure}

{

\subsection{Experiments on blast furnace dataset}\label{experiments on blast furnace dataset}

{In this section, we verify our results on real blast furnace data. The dataset are collected from two typical Chinese blast furnaces with the inner volume of about $2500$ $m^{3}$ and $750$ $m^{3}$, referred as blast furnace (a) and (b), respectively. Table \ref{tab1} lists the features that are relevant for predicting the silicon class labels for these two blast furnaces. We labeled the records satisfying ($x_6 \leq 0.4132$) / ($x_6 \leq 0.3736$) for furnace (a) / (b) as $-1$, and $+1$ otherwise \cite{Chen2019,gao2013rule}. Due to the (2-8h) time delay for furnace outputs to respond to inputs, we also treat 4 lagged terms for the first five features as inputs \cite{chen2021transfer}. 
Furnaces (a) and (b) have 794 and 800 samples, respectively. }

\begin{table}[htp]
\centering
\caption{ Input variables For blast furnaces}
\label{tab1}
\begin{threeparttable}
\begin{tabular}{lll}
\toprule
Variable name [Unit] & Symbol & Input variable \\
\midrule
Blast volume [m$^{3}$/min] & $x_{1}$ & $q^{-1}$, $q^{-2}$, $q^{-3}$, $q^{-4}$\\
Blast temperature [$^{\circ}$C] & $x_{2}$ & $q^{-1}$, $q^{-2}$, $q^{-3}$, $q^{-4}$ \\
Feed speed [mm/h] & $x_{3}$ & $q^{-1}$, $q^{-2}$, $q^{-3}$, $q^{-4}$ \\
Gas permeability [m$^{3}$/min$\cdot$kPa]  & $x_{4}$ & $q^{-1}$, $q^{-2}$, $q^{-3}$, $q^{-4}$ \\
Pulverized coal injection [ton] & $x_{5}$ & $q^{-1}$, $q^{-2}$, $q^{-3}$, $q^{-4}$ \\
Silicon content [wt\%] & $x_{6}$ & $q^{-1}$\\
\bottomrule
\end{tabular}
 \begin{tablenotes}
        \footnotesize
        \item[] $q^{-1}$, $q^{-2}$, $q^{-3}$, $q^{-4}$ represent delay operators, e.g., $q^{-1}x(t)=x(t-1)$. 
      \end{tablenotes}
    \end{threeparttable}
\end{table}

{We apply the S-SVM, M-SVM (with different parameters), and unregularized M-SVM using a linear kernel to the silicon classification problems. The misclassification error (the Bayes classifier is unknown here) is summarized in \Cref{fig3} where each dot represents the average of $10000$ random splits of the datasets with $m_1=100, 110, 120, \cdots,490$ samples as the training set and the remaining $300$ ones as the test set. For each partition, we normalize the variables of training samples to zero mean and unit variance, while the test samples are normalized accordingly. We see that the performance curves of the S-SVM and M-SVM are separable initially, then they get closer gradually and overlap finally, indicating that they have the same convergence rate but different PCR factors. 
Moreover, the improvement of PCR factors in MTL is more significant when the data size is relatively small.	
This phenomenon excellently agrees with the results of the simulation studies above. In addition, we observe from \Cref{fig3} that the unregularized M-SVM initially outperforms the M-SVM, due to the intrinsic similarity of these two furnaces. However, the unregularized M-SVM performs worse than the S-SVM and M-SVM when the training data size is relatively large. 
This occurs because the unregularized M-SVM classifier cannot converge to the corresponding Bayes rule without the regularization terms.}  


\begin{figure}[!htbp]
	\centering
	\includegraphics[width=0.45\textwidth]{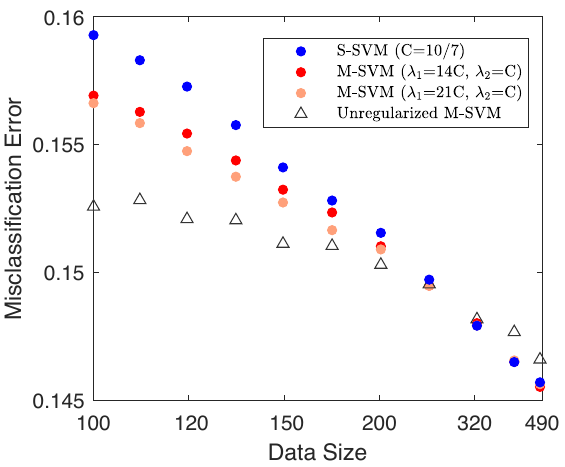}
	\captionsetup{singlelinecheck = false, justification=justified}
	\caption{This figure shows the misclassification error plots for the S-SVM, M-SVM, and unregularized M-SVM on blast furnace data. The horizontal axis is given in $1/m_1$ scale and is sorted in ascending order. We see that the plots of the S-SVM and M-SVM are initially separable, then they gradually get closer and finally overlap, indicating that they have the same convergence rate but the different PCR factors. Without the regularization terms, the unregularized M-SVM classifier cannot converge to the corresponding Bayes rule.}
	\label{fig3}
\end{figure}

\section{Discussion}\label{section 5}	
	There are a number of future works we are currently pursuing related to the error analysis of the regularized MTL.
	
	
	{
{We provided both the theoretical (in a one-dimensional least-square classification setting) and experimental justification to support our claim that the benefit of MTL is the improvement the PCR factor rather than the convergence rate. Based on the theoretical results of this simple case, we also discussed the conditions under which the MTL technique originating from \cite{evgeniou2004regularized} can improve the PCR factor.
	It would be interesting to theoretical verify this claim in more general MTL frameworks (e.g., the MTL methods listed in introduction), which requires a more elaborate analysis of the central limit property of the learning methods. We leave this issue for future work.
    }
	
	Another interesting question is whether we can still get optimal asymptotic performance (i.e., Bayes risk consistency) in the MTL setting discussed in this paper by assuming there are some common structures of functional space $\mathcal{ H }$ for each task. Maurer et al. \cite{maurer2016benefit} considered a similar problem in the context of an un-regularized MTL problem where a common feature representation of different tasks is shared. They also derived the error upper bound for this learning method. It is unknown whether sharing a common feature representation in MTL can achieve optimal performance.
	
	Intuitively, learning multiple tasks simultaneously in a model may interrupt the learning process of each task when we enforce the structural information to be shared between the classifiers. For example,  by considering to decompose an infinite dimensional functional space $\mathcal{ H }$ into the directed sum of two sub-spaces $\mathcal{H}_1$ and $\mathcal{H}_2$, that is $\mathcal{H}=\mathcal{H}_1+\mathcal{H}_2$, and assuming the function $f_t$ for each task $t$ can be written as $f_t=f_0+g_t$ where $f_0\in\mathcal{H}_1$ and $g_t\in\mathcal{H}_2$, for every $t\in\{1,\cdots, T\}$. In this case, the feasible space of functions for all $T$ tasks $(f_0, g_1, g_2, \cdots, g_T)$ is $\mathcal{H}_1\times \prod_{i=1}^{T} \mathcal { H }_2$, rather than $\prod_{i=1}^{T} \mathcal { H }$ considered in Eq. \eqref{svm2}. Thus, the classifier learned by MTL using the common functional structure technique {may} not converges to the Bayes rule in the limit of large data size, due to the smaller feasible functional space. The discussion indicates that it is important to balance the complexity of functional space and the computational costs in MTL. 
	
	{Note that we can reformulate M-SVM \eqref{svm2} as a standard SVM problem with $N$ training data \cite{evgeniou2004regularized}. If solving a SVM requires the time $\mathcal{O}(m_t^2)$ \cite{chang2011libsvm} for each task, the complexity of M-SVM would be $\mathcal{O}(N^2)$, where $N=\sum_{t=1}^{T}m_t$. Therefore, it is necessary to speedup M-SVM by using a parallelized computational structure when encountering the big data. To achieve this, we need to properly decompose the optimization problem \eqref{svm2} into $T$ subproblems for the $T$ tasks. For example, we can adapt a decomposing method from \cite{zhang2015parallel} to design the parallel algorithm for M-SVM. However, it is out of the scope of this work and will be left to future work.}
	
	Another attractive problem is generalizing our analysis framework to the regularized online MTL problems, where the large-scale data arrive sequentially. Especially, it is interesting to analyze the stochastic generalization error bounds in terms of the step sizes and the approximation property of functional space when the various regularization terms are included in the functional iterations to model different task relations. 
	
	\section{Conclusions}\label{section 6}
	In this work, we analyze the M-SVM's asymptotic performance by focusing on the asymptotic behaviors of the learned classifier in terms of the excess misclassification error.
	We first generalize the regularized multi-task learning (MTL) model \cite{evgeniou2004regularized} by introducing the kernel map and the uncertainty of data frequency of tasks into it.
	Then, we show the upper bound of the excess misclassification error to be $\mathcal{O}(( \log N ) ^ { - s/4 })$ almost surely. 
	This result demonstrates that the regularized MTL framework can produce reliable classification rules when the sample size goes to infinity. 
	{
    Furthermore, we find that the interaction of tasks vanishes as the data size goes to infinity, and the convergence rates of the M-SVM and its single-task counterpart have the same upper bound.
    The former suggests that M-SVM cannot improve the limit classifier's performance; based on the latter fact, we raise the conjecture that the optimal convergence rate is not improved either when the number of tasks is fixed. 
    
    {As a novel insight of MTL, our theoretical and experimental results achieved an excellent agreement that the benefit of M-SVM lies in improving the pre-convergence-rate (PCR) factor (denoted as the ratio between the excess misclassification error and its real convergence rate in Section III) rather than the convergence rate. This improvement of PCR factors is more significant when the data size is small. Moreover,
    our simulation results of several other MTL methods, including L{21}, Lasso, and SRMTL (these models can be found and solved by \cite{zhou2011malsar}) also demonstrate the generality of this new insight in MTL.
    	Therefore, PCR factor is more suitable and accurate to depict the essential advantages of MTL when the task number is fixed. }

\section*{Appendix}

\setcounter{equation}{0}
\renewcommand{\theequation}{\thesection.\arabic{equation}}
\renewcommand\thesection{\Alph{subsection}}
\renewcommand\thelemma{\thesection.\arabic{lemma}}
\subsection{Proofs}

\begin{proof}[{\bf The proof of Lemma 2}]
	Notice that
	\begin{align}
	&\mathcal{E}(f ^z_ { 1}  ,\cdots ,f^z _ { T })-\mathcal{E}(f^*_ { 1}  ,\cdots ,f^*_ { T })\\
	=&\left\{\mathcal { E } \left(f ^z_ { 1}, \cdots ,f^z _ { T } \right) - \mathcal { E } _ { z } \left( f ^z_ { 1} ,\cdots ,f^z _ { T }\right)\right\}	\label{eq. sample error 1}\\
	&+\left\{\mathcal{E}_z(f^z_ { 1},\cdots ,f^z_ { T })+\frac{\rho_1}{N}  \sum _ { t = 1} ^ { T }\frac{m_t}{N} \| f^z_ { t } \|_{K} ^ { 2} \right. \notag \\
	&\quad\quad\left. + \frac{\rho_2}{N} \sum _ { t = 1} ^ { T } \frac{m_t}{N}\left\| f^z_ { t } -\sum _ { s = 1} ^ { T }\frac{m_s}{N} f^z _ { s } \right\|_{K}^ { 2}\right\}\label{eq. f1}\\
	&-\left\{\mathcal{E}_z(f^{\mathcal{H}}_ { 1},\cdots ,f^{\mathcal{H}}_ { T })+\frac{\rho _ { 1}}{N} \sum _ { t = 1} ^ { T }p(t)\| f^{\mathcal{H}}_ { t } \|_{K} ^ { 2} \right. \notag\\
	&\qquad \left. + \frac{\rho _ { 2}}{N}\sum _ { t = 1} ^ { T }p(t)\left\| f^{\mathcal{H}}_ { t } -\sum _ { s = 1} ^ { T }p(s) f^\mathcal{ H } _ { s } \right\|_{K}^ { 2}\right\}  \label{eq. f2}\\
	&+\left\{\mathcal { E }_z\left(f^{\mathcal{H}}_ { 1}, \cdots ,f^{\mathcal{H}} _ { T } \right) - \mathcal { E } \left( f^{\mathcal{H}}_ { 1} ,\cdots ,f^{\mathcal{H}} _ { T }\right)\right\}\label{eq. sample error 2}\\
	&+\Bigg\{\mathcal{E}(f^{\mathcal{H}}_ { 1}  ,\cdots ,f^{\mathcal{H}} _ { T })-\mathcal{E}(f^*_ { 1}  ,\cdots ,f^*_ { T })+\frac{\rho _ { 1}}{N} \sum _ { t = 1} ^ { T }p(t)\| f^{\mathcal{H}}_ { t } \|_{K} ^ { 2}   \notag\\
	&\qquad \left. + \frac{\rho _ { 2}}{N}\sum _ { t = 1} ^ { T }p(t)\left\| f^{\mathcal{H}}_ { t } -\sum _ { s = 1} ^ { T }p(s) f^\mathcal{ H } _ { s } \right\|_{K}^ { 2}\right\}\label{eq. regularization error}\\
	&-\frac{\rho _ { 1}}{N} \sum _ { t = 1} ^ { T }\frac{m_t}{N} \| f^z_ { t } \|_{K} ^ { 2} -\frac{\rho _ { 2}}{N}\sum _ { t = 1} ^ { T } \frac{m_t}{N}\left\| f^z_ { t } -\sum _ { s = 1} ^ { T }\frac{m_s}{N} f^z _ { s } \right\|_{K}^ { 2} \notag
	\end{align}	
	Moreover, there is
	\begin{align}
	&\eqref{eq. f2} \notag \\
	\leq& -\left\{\mathcal{E}_z(f^{\mathcal{H}}_ { 1},\cdots ,f^{\mathcal{H}}_ { T })+\frac{\rho _ { 1}}{N} \sum _ { t = 1} ^ { T }\frac{m_{t}}{N}\| f^{\mathcal{H}}_ { t } \|_{K} ^ { 2} \right. \notag \\
	&\left. +\frac{\rho _ {2}}{N}\sum _ { t = 1} ^ { T } \frac{m_t}{N}\left\| f^{\mathcal{H}}_ { t } -\sum _ { s = 1} ^ { T }\frac{m_{t}}{N} f^\mathcal{ H } _ { s } \right\|_{K}^ { 2}\right\} \notag\\
	&-\frac{\rho_{ 1}}{N}\sum_{ t = 1}^{T} \left(p(t)-\frac{m_{t}}{N}\right)  \| f^{\mathcal{H}}_ { t }\|_{K} ^ { 2}\notag\\
	& +\frac{\rho_2}{N}\sum_{t=1}^{T}\left(\frac{m_{t}}{N}-p(t)\right)\left\| f^{\mathcal{H}}_ { t } -\sum _ { s = 1} ^ { T }\frac{m_s}{N} f^{\mathcal{H}}_{s} \right\|^2_{K}\notag\\
	& +\frac{2\rho _ {2}}{N}\sum _ { t = 1} ^ { T }p(t)\left\| f^{\mathcal{H}}_ { t } -\sum _ { s = 1} ^ { T }\frac{m_s}{N} f^{\mathcal{H}}_{s} \right\|_{K}\left\| \sum _ { t = 1} ^ { T } \left( p(t) - \frac{m_t}{N}\right) f^{\mathcal{H}}_{t} \right\|_{K}\notag
	\end{align}
	Combining above inequality with \eqref{eq. f1}, we obtained
	\begin{eqnarray}
	\begin{aligned}
	&\eqref{eq. f1}+\eqref{eq. f2} \\
	\leq&  -\frac{\rho_{ 1}}{N}\sum_{ t = 1}^{T} \left(p(t)-\frac{m_{t}}{N}\right)\| f^{\mathcal{H}}_ { t } \|_{K} ^ { 2}  \notag\\
	& +\frac{\rho_2}{N}\sum_{t=1}^{T}\left(\frac{m_{t}}{N}-p(t)\right)\left\| f^{\mathcal{H}}_ { t } -\sum _ { s = 1} ^ { T }\frac{m_s}{N} f^{\mathcal{H}}_{s} \right\|^2_{K}\notag\\
	& +\frac{2\rho _ {2}}{N}\sum _ { t = 1} ^ { T }p(t)\left\| f^{\mathcal{H}}_ { t } -\sum _ { s = 1} ^ { T }\frac{m_s}{N} f^{\mathcal{H}}_{s} \right\|_{K}\left\| \sum _ { t = 1} ^ { T } \left( p(t) - \frac{m_t}{N}\right) f^{\mathcal{H}}_{t} \right\|_{K}\notag
	\end{aligned}
	\end{eqnarray}
	Denote the right hand side of above inequality as $\mathcal{F}(N,T,\rho_{ 1},\rho_{ 2})$, $\mathcal{S}(N,T)=\eqref{eq. sample error 1}+\eqref{eq. sample error 2}$, and $\mathcal{ D }(N,T,\rho_{ 1},\rho_{ 2})=\eqref{eq. regularization error}$, we obtain the result.
\end{proof}

Before showing the proof of \textit{Lemma 3}, we first provide some auxiliary lemmas which will be helpful for the later proof.
The first one bounds the minimizer of M-SVM and its average empirical loss.

\begin{lemma}\label{lemma 6}
	For functions $\ell(\cdot)$ and $f_{t}^{z}$, we have
	\begin{itemize}
		\item $\sum_{ i = 1}^{m_{t}} \frac{\ell \left({y_{i,t}}\cdot f^{z}_{t}({\mathbf{x}_{i,t}} ) \right)}{m_{t}} \leq \frac{N}{m_{t}}$
		\item $\|f_{t}^{z}\|_{K}\leq \frac{N}{\sqrt{\rho_{ 1}\cdot m_{t}}}$
	\end{itemize}
\end{lemma}
\begin{proof}
	Let's define an auxiliary function $\tilde{f}(\cdot)\equiv0$ which obviously lies in $\mathcal{ H }_{K}$
	Note that $\{f^{z}_{t}\}$ minimize the objective function \ref{svm2}.
	Therefore, we have the following relation
	\begin{eqnarray}
	\begin{aligned}
	&\sum _ { t = 1} ^ { T } \sum _ { i = 1} ^ { m_{t} }\ell ({y_{i,t}}\cdot f^{z}_{t}({\mathbf{x}_{i,t}} ) )+ \rho _ { 1} \sum _ { t = 1} ^ { T }\frac{m_t}{N} \| f^{z}_ { t } \| _{K}^ { 2} \\
	&+ \rho _ { 2}\sum _ { t = 1} ^ { T } \frac{m_t}{N}\left\| f ^{z}_ { t } -\sum _ { s = 1} ^ { T }\frac{m_s}{N} f^{z} _ { s } \right\|_{K}^ { 2} \notag\\
	\leq&
	\sum _ { t = 1} ^ { T } \sum _ { i = 1} ^ { m_{t} }\ell\left({y_{i,t}}\cdot\tilde{f}({\mathbf{x}_{i,t}} ) \right) + \rho _ { 1} \sum _ { t = 1} ^ { T }\frac{m_t}{N} \| \tilde{f} \| _{K}^ { 2} \\
	&+ \rho _ { 2}\sum _ { t = 1} ^ { T } \frac{m_t}{N}\left\| \tilde{f} -\sum _ { s = 1} ^ { T }\frac{m_s}{N} \tilde{f} \right\|_{K}^ { 2}. \notag
	\end{aligned}
	\end{eqnarray}
	Since the left hand side of above inequality is greater than or equals to $\sum_{ i = 1}^{m_{t}} {\ell \left({y_{i,t}}\cdot f^{z}_{t}({\mathbf{x}_{i,t}} ) \right)}$, where $t=1,...T$, and the right hand side of the inequality equals to $N$, we have the relation
	\begin{equation*}
	\sum_{ i = 1}^{m_{t}} {\ell \left({y_{i,t}}\cdot f^{z}_{t}({\mathbf{x}_{i,t}} ) \right)} \leq N,
	\quad \text{and} \quad		
	\sum_{ i = 1}^{m_{t}} \frac{\ell \left({y_{i,t}}\cdot f^{z}_{t}({\mathbf{x}_{i,t}} ) \right)}{m_{t}} \leq \frac{N}{m_{t}}.
	\end{equation*}
	Similarly, since $\rho_{ 1}\frac{m_{t}}{N}\|f_{t}^{z}\|_{K}^{2}$ is also less than or equals to the left hand side of the first inequality in this lemma, therefore, we have
	\begin{equation*}
	\rho_{ 1}\frac{m_{t}}{N}\|f_{t}^{z}\|_{K}^{2} \leq N,
	\quad \text{and} \quad		
	\|f_{t}^{z}\|_{K}\leq \frac{N}{\sqrt{\rho_{ 1}\cdot m_{t}}}.
	\end{equation*}
	Thus the result is shown.
\end{proof}

The second auxiliary lemma bounds the sampling error as $\mathcal{O}(\frac{1}{\sqrt{N}})$.

\begin{lemma}\label{lemma 7}
	For a sequence of bounded, independent, and identically distributed random variables $\{X_{i}\}$, there is
	\begin{equation*}
	\frac{\sum_{i=1}^{N} X_{i}}{N} - \mathbb{E}X_1= \mathcal{O} \left(N^{-\frac{1}{2}+\epsilon}\right)\quad \text{a.s.}
	\end{equation*}
	where $\epsilon$ is any positive constant.
\end{lemma}

\begin{proof}
It can be proven directly by applying the Borel-Carntelli lemma to the Hoeffding's inequality. 
\end{proof}

The third auxiliary lemma bounds the generalization error of the M-SVM classifier.

\begin{lemma}\label{lemma 8}
	$$\mathbb { E }_{(\mathcal{X},\mathcal{Y})\sim P_{t} }[ \ell \left(\mathcal{Y}\cdot f^z_{t}( \mathcal{X} ) \right) ] - \sum_{ i = 1}^{m_{t}} \frac{\ell \left({y_{i,t}}\cdot f^{z}_{t}({\mathbf{x}_{i,t}} ) \right)}{m_{t}}  \leq \mathcal{O} \left(N^{-\frac{1}{4}+\epsilon}\right). $$
\end{lemma}

\begin{proof}
	We first denote some notations as follow:
	\begin{eqnarray}
	\tilde{\mathcal{ E}}_{t}(f) &:=& \mathbb { E }_{(\mathcal{X},\mathcal{Y})\sim P_{t} }[ \ell (\mathcal{Y}\cdot f^z_{t}( \mathcal{X} ) ) ] \notag \\
	\tilde{\mathcal{ E }}_{t}^{z}(f) &:=& \sum_{ i = 1}^{m_{t}} \frac{\ell \left({y_{i,t}}\cdot f^{z}_{t}({\mathbf{x}_{i,t}} ) \right)}{m_{t}} \notag \\
	\mathcal{B}_{R} &:=& \left\{ f\in \mathcal{ H }_{K} ~|~ \|f\|_{K} \leq R  \right\} \notag
	\end{eqnarray}
	Let $R^{*} =\frac{N}{\sqrt{\rho_{1}m_{t}}}$, then by \textit{Lemma \ref{lemma 6}} $f^{z}_{t}\in B_{R^*}$. Together with the fact that $\tilde{\mathcal{ E }}_{t}^{z}(f_{z})\leq \frac{N}{m_{t}}$ (c.f \textit{Theorem \ref{lemma 6}}), for any $\alpha>0$ we have
	\begin{eqnarray}
	&&P\left\{
	\tilde{\mathcal{ E}}_{t}(f^{z}_{t}) - 	\tilde{\mathcal{ E }}_{t}^{z}(f^{z}_{t}) > 4\alpha\left(1+\frac{N}{m_{t}}\right) ~\Big|~
	m_{t}		
	\right\} \notag \\
	&\leq&
	P\left\{ \sup_{f\in\mathcal{B}_{R^{*}}}\frac{\tilde{\mathcal{ E}}_{t}(f) - 	\tilde{\mathcal{ E }}_{t}^{z}(f)}{1+	\tilde{\mathcal{ E }}_{t}^{z}(f)} > 4\alpha~\Big|~
	m_{t}		
	\right\}\notag\\
	&\leq& \mathcal{N}\left(\frac{\alpha}{R^{*}}\right) \exp\left\{ -\frac{m_{t}\cdot\alpha^{2}}{32(1+\kappa R^{*})}\right\} \label{inequality 1}
	\end{eqnarray}	
	where the last inequality follows from \text{lemma 5} in \cite{wu2006analysis}, and $\mathcal{N}\left(\frac{\alpha}{R^{*}}\right)$ is the covering number defined as the minimal number of balls with radius $\frac{\alpha}{R^{*}}$ to cover the unite ball in RKHS.
	According to \cite{zhou2002covering}, the covering number has the relation
	\begin{equation*}
	\log \mathcal{N} (\epsilon) \leq M\left(\log \frac{1}{\epsilon}\right)^{d+1}
	\end{equation*}
	where $M$ is a constant, and $d$ is the dimension of the data set space.
	By plugging above inequality and $\alpha=\alpha^{*}(N,m_{t})=N^{\epsilon}\cdot\sqrt{\frac{1+\kappa R^{*}}{m_{t}}}$ into \eqref{inequality 1}, we have
	\begin{eqnarray}
	&&P\left\{
	\tilde{\mathcal{ E}}_{t}(f^{z}_{t}) - 	\tilde{\mathcal{ E }}_{t}^{z}(f^{z}_{t}) > 4\left(1+\frac{N}{m_{t}}\right)\cdot\alpha^{*}(N,m_{t}) ~\Big|~
	m_{t}		
	\right\} \notag \\
	&\leq&
	\exp\left\{  M\left(\log \frac{N}{N^{\epsilon}\sqrt{\rho_{ 1}+\rho_{ 1}\kappa R^{*}}}\right)^{d+1} -  \frac{1}{32} N^{2\epsilon} \right\} \notag \\
	&\leq& \exp\left\{  M\left(\log \frac{N^{1-\epsilon}}{\sqrt{\rho_{ 1}}}\right)^{d+1} -  \frac{1}{32} N^{2\epsilon} \right\} \notag \\
	\end{eqnarray}
	where $\epsilon$ is arbitrary positive constant.
	Furthermore, by above inequality and some easy calculations we can obtain
	\begin{align}
	&P\left\{
	\tilde{\mathcal{ E}}_{t}(f^{z}_{t}) - 	\tilde{\mathcal{ E }}_{t}^{z}(f^{z}_{t}) > 4\left(1+\frac{N}{m_{t}}\right)\cdot\alpha^{*}(N,m_{t}) ~\Big|~
	m_{t}		
	\right\} \notag\\
	&=  \mathbb{ E }_{m_{t}} \left[	P\left\{
	\tilde{\mathcal{ E}}_{t}(f^{z}_{t}) - 	\tilde{\mathcal{ E }}_{t}^{z}(f^{z}_{t}) > 4\alpha\left(1+\frac{N}{m_{t}}\right) ~\Big|~
	m_{t}		
	\right\} \right]\notag\\
	&\leq  \mathbb{ E }_{m_{t}} \left[ \exp\left\{  M\left(\log \frac{N^{1-\epsilon}}{\sqrt{\rho_{ 1}}}\right)^{d+1} -  \frac{1}{32} N^{2\epsilon} \right\} \right]\notag\\
	&= o \left(N^{-2}\right), \notag
	\end{align}
	and, therefore,
	\begin{equation*}
	\begin{aligned}
	&\sum_{N=1}^{\infty}P\left\{
	\tilde{\mathcal{ E}}_{t}(f^{z}_{t}) - 	\tilde{\mathcal{ E }}_{t}^{z}(f^{z}_{t}) > 4\left(1+\frac{N}{m_{t}}\right) \right.  \left. \cdot\alpha^{*}(N,m_{t}) ~\Big|~
	m_{t}		
	\right\} \\
	&< +\infty.
	\end{aligned}
	\end{equation*}
	
	By Borel-Cantelli Lemma, we have the relation
	\begin{equation*}
	\limsup_{N\to\infty} \tilde{\mathcal{ E}}_{t}(f^{z}_{t}) - 	\tilde{\mathcal{ E }}_{t}^{z}(f^{z}_{t}) - 4\left(1+\frac{N}{m_{t}}\right) \cdot\alpha^{*}(N,m_{t}) \leq 0 \quad \text{a.s.}.
	\end{equation*}
	Besides, according to the definition of $\alpha^{*}(\cdot,\cdot)$ and the strong law of large numbers, we have $\left(1+\frac{N}{m_{t}}\right) \cdot\alpha^{*}(N,m_{t})= \mathcal{O} \left(N^{-1/4+\epsilon}\right)$.
	Thus, $\tilde{\mathcal{ E}}_{t}(f^{z}_{t}) - 	\tilde{\mathcal{ E }}_{t}^{z}(f^{z}_{t}) \leq \mathcal{O} \left(N^{-1/4+\epsilon}\right)$
	which proves the result.
\end{proof}

With these preparations, we can prove \textit{Lemma \ref{lemma3}}.

\begin{proof}[{\bf The proof of Lemma 3}]
	We first consider the first part of sample error,  $\mathcal { E } \left(f ^z_ { 1}, \cdots ,f^z _ { T } \right) - \mathcal { E } _ { z } \left( f ^z_ { 1} ,\cdots ,f^z _ { T }\right)$. By the definition and some easy calculations, we can obtain,
	\begin{align}
	&\mathcal { E } \left(f ^z_ { 1}, \cdots ,f^z _ { T } \right) - \mathcal { E } _ { z } \left( f ^z_ { 1} ,\cdots ,f^z _ { T }\right) \notag \\
	=&  \sum_{t=1}^{T} p(t) \Bigg(\mathbb { E }_{(\mathcal{X},\mathcal{Y})\sim P_{t} }[ \ell \left(\mathcal{Y}\cdot f^z_{t}( \mathcal{X} ) \right) ] \left. -\sum_{ i = 1}^{m_{t}} \frac{\ell \left({y_{i,t}}\cdot f^{z}_{t}({\mathbf{x}_{i,t}} ) \right)}{m_{t}} \right) \notag \\
	&+ \sum_{t=1}^{T}\left[ \left(p(t)-\frac{m_{t}}{N}\right) \sum_{ i = 1}^{m_{t}} \frac{\ell \left({y_{i,t}}\cdot f^{z}_{t}({\mathbf{x}_{i,t}} ) \right)}{m_{t}} \right]  \label{eq. lemma proof auxilliary 1} 
	\end{align}
	Since, for each task $t$, the random variable $m_{t}$ is a counting process with bounded and independent increments, therefore, by \textit{Lemma \ref{lemma 7}}, there holds
	\begin{equation}{\label{eq. lemma proof auxiliarry 2}}
	p(t)-\frac{m_{t}}{N} =  \mathbb{E} \left(\frac{m_{t}}{N}\right) - \frac{m_{t}}{N} = \mathcal{O} \left(N^{-\frac{1}{2}+\epsilon}\right) \quad \text{a.s.}.
	\end{equation}
	By plugging above relation and \textit{lemma \ref{lemma 8}} into \eqref{eq. lemma proof auxilliary 1}, this relation leads to
	\begin{align} {\label{eq. lemma3 proof 1}}
	&\mathcal { E } \left(f ^z_ { 1}, \cdots ,f^z _ { T } \right) - \mathcal { E } _ { z } \left( f ^z_ { 1} ,\cdots ,f^z _ { T }\right)
	=  \mathcal{O} \left(N^{-\frac{1}{4}+\epsilon}\right)\quad \text{a.s.}.
	\end{align}
	
	Then we consider the second part of sample error, $\mathcal { E } _ { z } \left( f^{\mathcal { H }}_{1}, \cdots, f^{\mathcal { H }}_{T} \right) - \mathcal { E } \left( f^{\mathcal { H }}_{1}, \cdots, f^{\mathcal { H }}_{T} \right) $.
	Still by the definition and some easy calculations, we can obtain,
	\begin{align}
	& \mathcal { E } _ { z } \left( f^{\mathcal { H }}_{1}, \cdots, f^{\mathcal { H }}_{T} \right) - \mathcal { E } \left( f^{\mathcal { H }}_{1}, \cdots, f^{\mathcal { H }}_{T} \right) \notag \\
	=& \sum_{t=1}^{T} \left[ \left(\frac{m_{t}}{N}-p(t)\right)\sum_{ i = 1}^{m_{t}} \frac{\ell \left({y_{i,t}}\cdot f^{\mathcal{ H }}_{t}({\mathbf{x}_{i,t}} ) \right)}{m_{t}} \right] \notag \\
	&  +\sum_{t=1}^{T} p(t) \left[\sum_{ i = 1}^{m_{t}} \frac{\ell \left({y_{i,t}}\cdot f^{\mathcal{ H }}_{t}({\mathbf{x}_{i,t}} ) \right)}{m_{t}} \right. \notag\\
	& - \mathbb { E }_{(\mathcal{X},\mathcal{Y})\sim P_{t} }[ \ell \left(\mathcal{Y}\cdot f^\mathcal{H}_{t}( \mathcal{X} ) \right) ] \Bigg ]  \label{eq. lemma3 proof auxillary 3}
	\end{align}
	Moreover, since $\ell \left({y_{i,t}}\cdot f^{\mathcal{ H }}_{t}({\mathbf{x}_{i,t}} )\right) =\max\{0, 1-{y_{i,t}}f^{\mathcal{ H }}_{t}( {\mathbf{x}_{i,t}} )\}  \leq 1+\|f^{\mathcal{ H }}_{t}\|_{\infty} \leq 1+\kappa\|f^{\mathcal{ H }}_{t}\|_{K}$ and $\|f^{\mathcal{ H }}_{t}\|_{K}$ is finite, for each data sample $\ell \left({y_{i,t}}\cdot f^{\mathcal{ H }}_{t}({\mathbf{x}_{i,t}} )\right) $ is uniformly bounded by a constant $R$ where $R=1+\kappa\|f^{\mathcal{ H }}_{t}\|_{K}$.
	Therefore, by \eqref{eq. lemma proof auxiliarry 2} and the uniform boundedness of $\ell \left({y_{i,t}}\cdot f^{\mathcal{ H }}_{t}({\mathbf{x}_{i,t}} )\right)$, there holds
	\begin{equation}{\label{eq. lemma3 proof auxillary 4}}
	\begin{aligned}
	&\sum_{t=1}^{T} \left[ \left(\frac{m_{t}}{N}-p(t)\right)\sum_{ i = 1}^{m_{t}} \frac{\ell \left({y_{i,t}}\cdot f^{\mathcal{ H }}_{t}({\mathbf{x}_{i,t}} ) \right)}{m_{t}} \right]\\
	&\leq \sum_{t=1}^{T} \left[ \left(\frac{m_{t}}{N}-p(t)\right)\sum_{ i = 1}^{m_{t}} \frac{R }{m_{t}} \right]
	= \mathcal{O} \left(N^{-\frac{1}{2}+\epsilon}\right)\quad \text{a.s.}.
	\end{aligned}
	\end{equation}
	Besides, due to i.i.d. property of sampled data $(\mathbf{x}_{i,t},y_{i,t})$ for each task $t$, the random variable $\ell \left({y_{i,t}}\cdot f^{\mathcal{ H }}_{t}({\mathbf{x}_{i,t}} )\right) $ is also independent and identically distributed.
	Together with the boundedness of $\ell \left({y_{i,t}}\cdot f^{\mathcal{ H }}_{t}({\mathbf{x}_{i,t}} )\right) $, we can apply \textit{Lemma \ref{lemma 7}} to $\ell \left({y_{i,t}}\cdot f^{\mathcal{ H }}_{t}({\mathbf{x}_{i,t}} )\right) $ and obtain $\sum_{ i = 1}^{m_{t}} \frac{\ell \left({y_{i,t}}\cdot f^{\mathcal{ H }}_{t}({\mathbf{x}_{i,t}} ) \right)}{m_{t}}
	- \mathbb { E }_{(\mathcal{X},\mathcal{Y})\sim P_{t} }[ \ell \left(\mathcal{Y}\cdot f^\mathcal{H}_{t}( \mathcal{X} ) \right) ] = \mathcal{O} \left(N^{-\frac{1}{2}+\epsilon}\right)$ a.s..
	By plugging this equation and \eqref{eq. lemma3 proof auxillary 4} into the last two lines of \eqref{eq. lemma3 proof auxillary 3}, we obtain
	\begin{equation}{\label{eq. lemma3 proof eq2}}
	\mathcal { E } _ { z } \left( f^{\mathcal { H }}_{1}, \cdots, f^{\mathcal { H }}_{T} \right) - \mathcal { E } \left( f^{\mathcal { H }}_{1}, \cdots, f^{\mathcal { H }}_{T} \right) = \mathcal{O} \left(N^{-\frac{1}{2}+\epsilon}\right)\quad \text{a.s.}.
	\end{equation}
	
	Finally combining \eqref{eq. lemma3 proof 1} and \eqref{eq. lemma3 proof eq2}, we have $\mathcal { S } (N,T)= \mathcal{O} \left(N^{-\frac{1}{4}+\epsilon}\right)$ a.s., 
	which proves this lemma.
\end{proof}

\begin{proof}[{\bf The proof of Lemma 4}]
	First, we note that function $\ell(x)=\max\{0,1-x\}$ is a Lipschitz function, satisfying $|\ell(x)-\ell(y)|\leq |x-y|$ where $x$ and $y$ are arbitrary real numbers.
	Therefore, for any $h_{1},~h_{2},\dots,h_{T} \in \mathcal{ H }_{K}$ there is
	\begin{eqnarray}
	\begin{aligned}
	&\mathcal { E } \left(h_1, \cdots, h_{T} \right) - \mathcal { E } \left( f^{*}_{1}, \cdots, f^{*}_{T} \right)\notag \\
	&\leq \sum_{i=1}^{T}p(t) \mathbb{ E } _{(\mathcal{X},\mathcal{Y})\sim P_{t} } \left|  h_{t}( \mathcal{X} ) ) -f^*_{t}( \mathcal{X} ) ) \right|  \notag \\
	&\leq \sum_{i=1}^{T}p(t) \left(\mathbb{ E } _{(\mathcal{X},\mathcal{Y})\sim P_{t} } \left|  h_{t}( \mathcal{X} ) ) -f^*_{t}( \mathcal{X} ) ) \right|^{2} \right)^{1/2}  \notag
	\end{aligned}
	\end{eqnarray}
	where the first inequality follows from Lipschitz continuity and the second one follows from Holder inequality.
	For simplicity, we denote $\|h_{t}-f_{t}\|_{\mathcal{L}^{2}_{P_{t}}}:= \left(\mathbb{ E } _{(\mathcal{X},\mathcal{Y})\sim P_{t} } \left|  h_{t}( \mathcal{X} ) ) -f^*_{t}( \mathcal{X} ) ) \right|^{2} \right)^{1/2}$.
	Through plugging above the last inequality to $D(N,T,\rho_{ 1},\rho_{ 2})$, we have
	\begin{eqnarray}
	\begin{aligned}
	&D(N,T,\rho_{ 1},\rho_{ 2}) \notag \\
	\leq&
	\inf_{ h _ { 1}, \cdots ,h _ { T }\in\mathcal { H } _{K} }
	\left\{
	\sum_{ t = 1}^{T} p(t) \|h_{t}-f^{*}_{t}\|_{\mathcal{L}^{2}_{P_{t}}} \notag +\frac{\rho _ { 1}}{N}\sum _ { t = 1} ^ { T }p(t)\| f^{*}_ { t } \|_{K} ^ { 2} \notag \right. \\
	& \left. +\frac{\rho _ { 2}}{N}\sum _ { t = 1} ^ { T } p(t)\left\| h_ { t } -\sum _ { s = 1} ^ { T }p(s) h _ { s } \right\|_{K}^ { 2}
	\right\} \notag   	 \\
	\leq& \inf_{R>0} \left\{
	\inf_{ \tiny\begin{array}{c}
		h _ { 1}, \cdots ,h _ { T }\in\mathcal { H } _{K}\\
		\|h_{t}\|_{K}\leq R
		\end{array} }
	\left\{
	\sum_{ t = 1}^{T} p(t) \|h_{t}-f^{*}_{t}\|_{\mathcal{L}^{2}_{P_{t}}, } 	 	 \right\} \notag \right. \\
	&+\frac{\rho _ { 1} R^{2}}{N}+ \frac{2\rho _ { 2}R^{2}}{N}
	\Bigg\} \notag \\
	\leq& \inf_{R>0} \left\{
	C_{0}C_{s}\left(\log R \right)^{-s/4}
	+\frac{\rho _ { 1}  R^{2}}{N}+ \frac{2\rho _ { 2}R^{2}}{N}
	\right\} \notag
	\end{aligned}
	\end{eqnarray}
	where the last inequality follows immediately from \cite{smale2003estimating,zhou2013density}.
	In above, $C_{0}$ and $s$ are two positive constants, while $C_{s}$ is another constant depends on $s$.
	By choosing $R=N^{1/2-\epsilon}$ where $\epsilon>0$, we have $D(N,T,\rho_{ 1},\rho_{ 2}) \leq C_{0}C_{s}\left( (1/2-\epsilon)\log N \right)^{-s/4}
	+{\rho _ { 1}  N^{-\epsilon}}+{2\rho _ { 2}N^{-\epsilon}} = \mathcal{O} \left(\log N\right)^{-s/4},$
	which proves the result.
\end{proof}

\begin{proof}[{\bf The proof of Lemma 5}]
	By the definition and some easy calculation, we get
	\begin{eqnarray}
	\begin{aligned}
	&\mathcal{F}(N,T,\rho_1, \rho_2)\\
	\leq& -\frac{\rho_{ 1}}{N}\sum_{ i = 1}^{T} \left(p(t)-\frac{m_{t}}{N}\right)\| f^{\mathcal{H}}_ { t } \|_{K} ^ { 2}  \notag\\
	&+\frac{\rho _ {2}}{N}
	\sum_{t=1}^{T}\left\|p(t)-\frac{m_{t}}{N}\right\|\left(\|f^{\mathcal{H}}_ { t } \|_{K}+\sum _ { s = 1} ^ { T } \left\| f^{\mathcal{H}}_{s}\right\|_{K}\right)^2\notag\\
	&+\frac{2\rho _ {2}}{N} \left(\sum_{t=1}^{T} \left\|p(t)-\frac{m_{t}}{N}\right\| \cdot \left\|f_{t}^\mathcal{H}\right\|_{K} \right) \notag\\
	&\times	\left( \sum _ { t = 1} ^ { T } \left\| f^{\mathcal{H}}_ { t } \right\|_{K}+ \sum _ { s = 1} ^ { T } \left\| f^{\mathcal{H}}_{s} \right\|_{K} \right)\notag
	\end{aligned}
	\end{eqnarray}
	where the terms $\frac{\rho_1}{N}$, $\frac{\rho_2}{N}$, and $p(t)-\frac{m_{t}}{N}$ determine the convergent rate of the frequency error.
	Since, for each task $t$, the random variable $m_{t}$ is a counting process with bounded and independent increments, therefore, by \textit{Lemma \ref{lemma 7}}, there holds $	p(t)-\frac{m_{t}}{N} =  \mathbb{E} \left(\frac{m_{t}}{N}\right) - \frac{m_{t}}{N} = \mathcal{O} \left(N^{-\frac{1}{2}+\epsilon}\right) \quad \text{a.s.}.$
	Therefore, $\mathcal{F}(T,N,\rho_1, \rho_2)\leq  \mathcal{O} \left(N^{-\frac{3}{2}+\epsilon}\right) ~ \text{a.s.}$, which shows the result.
\end{proof}

\bibliographystyle{ieeetr}
\bibliography{reference}

\vspace{-10 mm}
\begin{IEEEbiography}[{\includegraphics[width=1in,height=1.25in,clip,keepaspectratio]{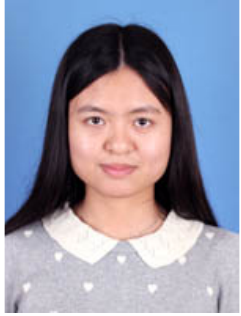}}]
	{Shaohan Chen} received the B.Sc. degrees in Mathematics and Applied Mathematics from Jimei University, China, in 2014.
	She is currently working towards the Ph.D. degree in operational research and cybernetics at Zhejiang University.
	
	Her research interests are in the areas of machine learning models design and interpretations.
\end{IEEEbiography}

\vspace{-10 mm}
\begin{IEEEbiography}[{\includegraphics[width=1in,height=1.25in,clip,keepaspectratio]{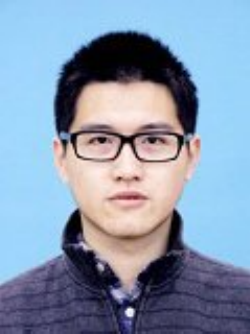}}]{Zhou Fang}
	received the B.Sc. and Ph.D. degrees in mathematics from Zhejiang university, China, in 2014 and 2019, respectively. 
	Since November 2019, he has been a postdoc with the department of biosystems science and engineering, at ETH-Z\"urich, Switzerland.
	
	His research interests are in the areas of learning theory, control theory, and systems/synthetic biology.  
\end{IEEEbiography}

\vspace{-10 mm}
\begin{IEEEbiography}[{\includegraphics[width=1in,height=1.25in,clip,keepaspectratio]{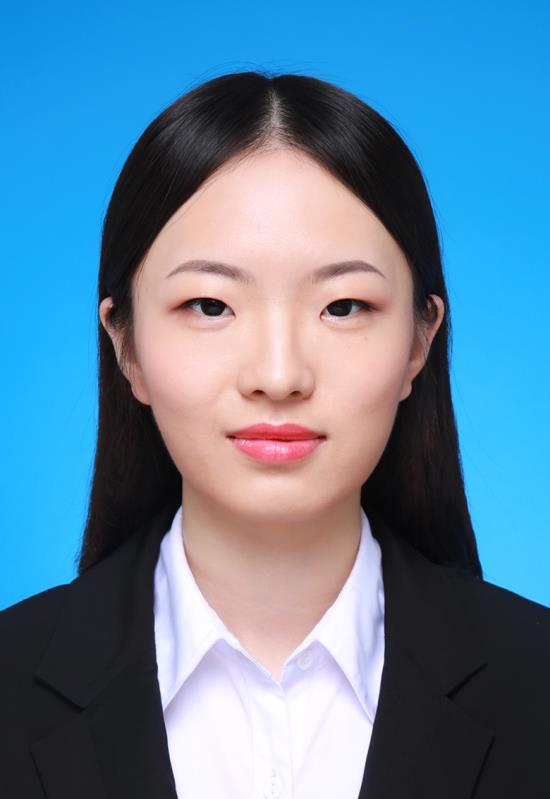}}]
	{Sijie Lu} received the B.Sc. degrees in Mathematics from China agricultural university, China, in 2019.
	She is currently working towards the MS.C. degree in operational research and cybernetics at Zhejiang University.
	Her research interests are in the areas of machine learning and its industrial applications.
\end{IEEEbiography}

\vspace{-10 mm}
\begin{IEEEbiography}[{\includegraphics[width=1in,height=1.25in,clip,keepaspectratio]{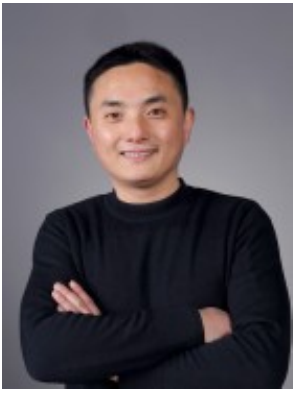}}]
{Chuanhou Gao} received the B.Sc.
degrees in chemical engineering from Zhejiang University
of Technology, China, in 1998, and the Ph.D.
degrees in operational research and cybernetics from
Zhejiang University, China, in 2004.
From June 2004 until May 2006, he was a Postdoctor
in the Department of Control Science and Engineering
at Zhejiang University. Since June 2006,
he has been with the Department of Mathematics at
Zhejiang University, where he is currently a full
Professor. 
His research interests are in the areas of
data-driven modeling, control and optimization, problems-driven applied mathematics,
machine learning and transparency of black-box modeling techniques.

Dr. Gao was a Guest Editor of IEEE TRANSACTIONS ON INDUSTRIAL
INFORMATICS, ISIJ International and Journal of Applied Mathematics and an editor of Metallurgical Industry Automation from May 2013.
Currently, he is an associate editor of IEEE Transactions on Automatic Control.
\end{IEEEbiography}

\end{document}


\title{Supplementary material from ``Asymptotic performance of regularized multi-task learning based on SVM models"}



\author{Shaohan~Chen,~Zhou~Fang,~Sijie~Lu,~and~Chuanhou~Gao,~\IEEEmembership{Senior Member,~IEEE}
\thanks{
This work was supported by
the National Natural Science Foundation of China under Grant No.
11671418 and the Zhejiang Provincial Natural Science Foundation of China under Grant No. LZ20A010002.}
\thanks{S. Chen, S. Lu and C. Gao are with the School of Mathematical Sciences, Zhejiang
University, Hangzhou 310027, China (Corresponding e-mail: gaochou@zju.edu.cn (C. Gao)).}
\thanks{Z. Fang is with the Department of Biosystems Science and Engineering, ETH Zurich, Switzerland.}}



\maketitle

\IEEEdisplaynontitleabstractindextext

\IEEEpeerreviewmaketitle

\section*{Appendix}

\setcounter{equation}{0}
\renewcommand{\theequation}{\thesection.\arabic{equation}}
\renewcommand\thesection{\Alph{subsection}}
\renewcommand\thelemma{\thesection.\arabic{lemma}}

In this appendix, we are going to provide a brief introduction to the Reproducing Kernel Hilbert Space, and give the proofs of \textit{Lemma 1--5} and of \textit{Theorem 1} in our article.
The main scheme and most of techniques are quite standard in the learning theory literature.

\subsection{Reproducing Kernel Hilbert Space}
For complete, we introduce the definition of the Reproducing Kernel Hilbert Space (RKHS) $\mathcal{ H }_K$  following the notations and terminologies in literature \cite{wu2006analysis}.
(RKHS) $\mathcal{ H }_K$ associated with the kernel $K$ to be the completion of the linear span of the set of functions $\left\{ K _ { x } : = K ( x ,\cdot ) : x \in X \right\}$ with the inner product $\langle \cdot ,\cdot \rangle _ { \mathcal { H } _ { K } } = \langle \cdot ,\cdot \rangle _ { K }$ satisfying
\begin{equation}
\begin{aligned}
\left\| \sum _ { i = 1} ^ { m } c _ { i } K _ { x _ { i } } \right\| _ { K } ^ { 2} &= \left\langle \sum _ { i = 1} ^ { m } c _ { i } K _ { x _ { i } } ,\sum _ { i = 1} ^ { m } c _ { i } K _ { x _ { i } } \right\rangle _ { K } \\
&= \sum _ { i ,j = 1} ^ { m } c _ { i } K \left( x _ { i } ,x _ { j } \right) c _ { j }.
\end{aligned}
\end{equation}
The reproducing property of RKHS $\mathcal{ H }_K$ is defined as
\begin{equation}\label{RKHS}
< K _ { x } ,g > _ { K } = g ( x ) ,\quad \forall x \in X ,g \in \mathcal { H } _ { K }.
\end{equation}
Denote $C(X)$ to be the space of continuous functions on $X$ with the norm $\| \cdot \| _ { \infty }$. Then Eq. \eqref{RKHS} leads to $\| g \| _ { \infty } \leq \kappa \| g \| _ { K } ,\quad \forall g \in \mathcal { H } _ { K }$, where $\kappa = \sup _ { x \in X } \sqrt { K ( x ,x ) }$. This means $\mathcal{ H }_K$ can be embedded into $C(X)$.

\subsection{Proofs}

\begin{proof}[\bf The proof of Lemma 1]
	Notice that the Lagrangian function for Eq. \eqref{svm1} is given by the following formula
	For convenience, we denote the objective function of Eq. \eqref{svm1} as $J(f_0, g_t, \xi_{it})$.
	\begin{eqnarray}
	\begin{aligned}
	L(f_0, g_t,\alpha_{it}, \gamma_{it})\notag 
	=&J(f_0, g_t, \xi_{it})
	- \sum _ { t = 1} ^ { T } \sum _ { i = 1} ^ { m_t } \gamma _ { i i } \xi _ { i t } \notag\\
	 &- \sum _ { t = 1} ^ { T } \sum _ { i = 1} ^ { m_t } \alpha _ { i t } \left( y _ { i t } \left( f_ { 0} + g _ { t } \right)(\mathbf{x} _ { i t } - 1+ \xi _ { i t } \right) \notag \\
	&
	\end{aligned}
	\end{eqnarray}
	where $\alpha_{it}, \gamma_{it}$ are nonnegative Lagrangian multipliers. Setting the derivative of $L$ with respect to $f_0$ to zero and by considering $f_0,g_t\in \mathcal{ H }_K$ gives the equation
	$\bar{f}_ { 0} ^ { *} = \frac { 1} { 2\lambda _ { 2} } \sum _ { t = 1} ^ { T } \sum _ { i = 1} ^ { m_t } \alpha _ { i t } y _ { i t } K_{x _ { i t}}$.
	Similar to the $g_t$, for every $t\in\{1, \cdots, T\}$, we have
	$g_ { t} = \frac { N } { 2\lambda _ { 1}m_t } \sum _ { i = 1} ^ { m_t } \alpha _ { i t } y _ { i t } K_{x _ { i t}}$.
	By combining these two equations and together with the equation $\bar{f}^{*}_t=\bar{f}^*_0+\bar{g}^*_t$ we can easily obtain the result.
\end{proof}

\begin{proof}[{\bf The proof of Theorem 1}]
	Using the result of Lemma \ref{lemma1} and $\bar{f}^*_t=\bar{f}^*_0+\bar{g}^*_t$, we have the following equation
	\begin{eqnarray}
	\begin{aligned}
	&\lambda _ { 1} \sum _ { t = 1} ^ { T }\frac{m_t}{N} \| \bar{g}^*_ { t } \|_K ^ { 2} + \lambda _ { 2}\| \bar{f}^*_ {0} \|_K ^ { 2}\\
	&= \lambda _ { 1} \sum _ { t = 1} ^ { T }\frac{m_t}{N} \| \bar{f}^{*}_ { t } \|_K ^ { 2}- \frac { \lambda _ { 1} ^ { 2} } { \lambda _ { 1} + \lambda _ { 2} } \| \sum _ { t = 1} ^ { T } \frac{m_t}{N} \bar{f}^* _ { t } \|_K^{2}. \notag
	\end{aligned}
	\end{eqnarray}
	On the other hand we have
	\begin{eqnarray}
	\begin{aligned}
	&\rho _ { 1} \sum _ { t = 1} ^ { T } \frac{m_t}{N} \| \bar{f}^{*} _ { t } \|_K ^ { 2}+\rho _ { 1} \sum _ { t= 1} ^ { T } \frac{m_t}{N}\| \bar{f}^{*}_t - \sum _ { s = 1} ^ { T } \frac{m_s}{N} \bar{f}^{*} _ { s} \|_K^{2}\\
	&= (\rho _ { 1}+\rho_{ 2}) \sum _ { t = 1} ^ { T } \frac{m_t}{N} \| \bar{f}^* _ { t } \|_K ^ { 2}-\rho _ { 2}\| \sum _ { s = 1} ^ { T } \frac{m_s}{N} \bar{f}^{*} _ { s} \|_K^{2}. \notag
	\end{aligned}
	\end{eqnarray}
	Comparing the above two equations we can obtain
	$\rho _ { 1} = \frac { \lambda _ { 1} \lambda _ { 2} } { \lambda _ { 1} + \lambda _ { 2} }$ and $\rho _ { 2} = \frac { \lambda _ { 1} ^ { 2} } { \lambda _ { 1} + \lambda _ { 2} }$.
	Thus the results hold.
\end{proof}

\begin{proof}[{\bf The proof of Lemma 2}]
	Notice that
	\begin{align}
	&\mathcal{E}(f ^z_ { 1}  ,\cdots ,f^z _ { T })-\mathcal{E}(f^*_ { 1}  ,\cdots ,f^*_ { T })\\
	=&\left\{\mathcal { E } \left(f ^z_ { 1}, \cdots ,f^z _ { T } \right) - \mathcal { E } _ { z } \left( f ^z_ { 1} ,\cdots ,f^z _ { T }\right)\right\}	\label{eq. sample error 1}\\
	&+\left\{\mathcal{E}_z(f^z_ { 1},\cdots ,f^z_ { T })+\frac{\rho_1}{N}  \sum _ { t = 1} ^ { T }\frac{m_t}{N} \| f^z_ { t } \|_{K} ^ { 2} \right. \notag \\
	&\left. + \frac{\rho_2}{N} \sum _ { t = 1} ^ { T } \left\| f^z_ { t } -\sum _ { s = 1} ^ { T }\frac{m_s}{N} f^z _ { s } \right\|_{K}^ { 2}\right\}\label{eq. f1}\\
	&-\left\{\mathcal{E}_z(f^{\mathcal{H}}_ { 1},\cdots ,f^{\mathcal{H}}_ { T })+\frac{\rho _ { 1}}{N} \sum _ { t = 1} ^ { T }p(t)\| f^{\mathcal{H}}_ { t } \|_{K} ^ { 2} \right. \notag\\
	&\left. + \frac{\rho _ { 2}}{N}\sum _ { t = 1} ^ { T } \left\| f^{\mathcal{H}}_ { t } -\sum _ { s = 1} ^ { T }p(s) f^\mathcal{ H } _ { s } \right\|_{K}^ { 2}\right\}  \label{eq. f2}\\
	&+\left\{\mathcal { E }_z\left(f^{\mathcal{H}}_ { 1}, \cdots ,f^{\mathcal{H}} _ { T } \right) - \mathcal { E } \left( f^{\mathcal{H}}_ { 1} ,\cdots ,f^{\mathcal{H}} _ { T }\right)\right\}\label{eq. sample error 2}\\
	&+\Bigg\{\mathcal{E}(f^{\mathcal{H}}_ { 1}  ,\cdots ,f^{\mathcal{H}} _ { T })-\mathcal{E}(f^*_ { 1}  ,\cdots ,f^*_ { T })  \notag\\
	&\left. +\frac{\rho _ { 1}}{N} \sum _ { t = 1} ^ { T }p(t)\| f^{\mathcal{H}}_ { t } \|_{K} ^ { 2} + \frac{\rho _ { 2}}{N}\sum _ { t = 1} ^ { T } \left\| f^{\mathcal{H}}_ { t } -\sum _ { s = 1} ^ { T }p(s) f^\mathcal{ H } _ { s } \right\|_{K}^ { 2}\right\}\label{eq. regularization error}\\
	&-\frac{\rho _ { 1}}{N} \sum _ { t = 1} ^ { T }\frac{m_t}{N} \| f^z_ { t } \|_{K} ^ { 2} -\frac{\rho _ { 2}}{N}\sum _ { t = 1} ^ { T } \left\| f^z_ { t } -\sum _ { s = 1} ^ { T }\frac{m_s}{N} f^z _ { s } \right\|_{K}^ { 2} \notag
	\end{align}	
	Moreover, there is.
	\begin{eqnarray}
	\begin{aligned}
	&\eqref{eq. f2}\\
	\leq& -\left\{\mathcal{E}_z(f^{\mathcal{H}}_ { 1},\cdots ,f^{\mathcal{H}}_ { T })+\frac{\rho _ { 1}}{N} \sum _ { t = 1} ^ { T }\frac{m_{t}}{N}\| f^{\mathcal{H}}_ { t } \|_{K} ^ { 2} \right. \\
	&\left. + \frac{\rho _ {2}}{N}\sum _ { t = 1} ^ { T } \left\| f^{\mathcal{H}}_ { t } -\sum _ { s = 1} ^ { T }\frac{m_{t}}{N} f^\mathcal{ H } _ { s } \right\|_{K}^ { 2}\right\} \notag\\
	&-\frac{\rho_{ 1}}{N}\sum_{ t = 1}^{T} \left(p(t)-\frac{m_{t}}{N}\right)  \| f^{\mathcal{H}}_ { t }\|_{K} ^ { 2}\notag\\
	&+\frac{2\rho _ {2}}{N}\left\| \sum _ { t = 1} ^ { T } \left( p(t) - \frac{m_t}{N}\right) f^{\mathcal{H}}_{t} \right\|_{K} \sum _ { t = 1} ^ { T } \left\| f^{\mathcal{H}}_ { t } -\sum _ { s = 1} ^ { T }\frac{m_s}{N} f^{\mathcal{H}}_{s} \right\|_{K}\notag
	\end{aligned}
	\end{eqnarray}
	Combining above inequality with \eqref{eq. f1}, we obtained
	\begin{eqnarray}
	\begin{aligned}
	&\eqref{eq. f1}+\eqref{eq. f2} \\
	\leq&  -\frac{\rho_{ 1}}{N}\sum_{ t = 1}^{T} \left(p(t)-\frac{m_{t}}{N}\right)\| f^{\mathcal{H}}_ { t } \|_{K} ^ { 2}  \notag\\
	& +\frac{2\rho _ {2}}{N}\left\| \sum _ { t = 1} ^ { T } \left( p(t) - \frac{m_t}{N}\right) f^{\mathcal{H}}_{t} \right\|_{K} \sum _ { t = 1} ^ { T } \left\| f^{\mathcal{H}}_ { t } -\sum _ { s = 1} ^ { T }\frac{m_s}{N} f^{\mathcal{H}}_{s} \right\|_{K}\notag
	\end{aligned}
	\end{eqnarray}
	Denote the right hand side of above inequality as $\mathcal{F}(N,T,\rho_{ 1},\rho_{ 2})$, $\mathcal{S}(N,T)=\eqref{eq. sample error 1}+\eqref{eq. sample error 2}$, and $\mathcal{ D }(N,T,\rho_{ 1},\rho_{ 2})=\eqref{eq. regularization error}$, we obtain the result.
\end{proof}

Before showing the proofs of \textit{Lemma 3}, we first induce some auxiliary lemmas which will be very helpful in the later proof.
The first auxiliary lemma  shows the upper bounds for $f^{z}_{t}$ and $\ell \left({y_{i,t}}\cdot f_{t}({\mathbf{x}_{i,t}} ) \right)$'s average.

\begin{lemma}\label{lemma 6}
	For functions $\ell(\cdot)$ and $f_{t}^{z}$, we have
	\begin{itemize}
		\item $\sum_{ i = 1}^{m_{t}} \frac{\ell \left({y_{i,t}}\cdot f^{z}_{t}({\mathbf{x}_{i,t}} ) \right)}{m_{t}} \leq \frac{N}{m_{t}}$
		\item $\|f_{t}^{z}\|_{K}\leq \frac{N}{\sqrt{\rho_{ 1}\cdot m_{t}}}$
	\end{itemize}
\end{lemma}
\begin{proof}
	Let's define an auxiliary function $\tilde{f}(\cdot)\equiv0$ which obviously lies in $\mathcal{ H }_{K}$
	Note that $\{f^{z}_{t}\}$ minimize the objective function \ref{svm2}.
	Therefore, we have the following relation
	\begin{eqnarray}
	\begin{aligned}
	&\sum _ { t = 1} ^ { T } \sum _ { i = 1} ^ { m_{t} }\ell ({y_{i,t}}\cdot f^{z}_{t}({\mathbf{x}_{i,t}} ) )+ \rho _ { 1} \sum _ { t = 1} ^ { T }\frac{m_t}{N} \| f^{z}_ { t } \| _{K}^ { 2} \\
	&+ \rho _ { 2}\sum _ { t = 1} ^ { T } \left\| f ^{z}_ { t } -\sum _ { s = 1} ^ { T }\frac{m_s}{N} f^{z} _ { s } \right\|_{K}^ { 2} \notag\\
	\leq&
	\sum _ { t = 1} ^ { T } \sum _ { i = 1} ^ { m_{t} }\ell\left({y_{i,t}}\cdot\tilde{f}({\mathbf{x}_{i,t}} ) \right) + \rho _ { 1} \sum _ { t = 1} ^ { T }\frac{m_t}{N} \| \tilde{f} \| _{K}^ { 2} \\
	&+ \rho _ { 2}\sum _ { t = 1} ^ { T } \left\| \tilde{f} -\sum _ { s = 1} ^ { T }\frac{m_s}{N} \tilde{f} \right\|_{K}^ { 2}. \notag
	\end{aligned}
	\end{eqnarray}
	Since the left hand side of above inequality is greater than or equals to $\sum_{ i = 1}^{m_{t}} {\ell \left({y_{i,t}}\cdot f^{z}_{t}({\mathbf{x}_{i,t}} ) \right)}$, where $t=1,...T$, and the right hand side of the inequality equals to $N$, we have the relation
	\begin{equation*}
	\sum_{ i = 1}^{m_{t}} {\ell \left({y_{i,t}}\cdot f^{z}_{t}({\mathbf{x}_{i,t}} ) \right)} \leq N,
	\quad \text{and} \quad		
	\sum_{ i = 1}^{m_{t}} \frac{\ell \left({y_{i,t}}\cdot f^{z}_{t}({\mathbf{x}_{i,t}} ) \right)}{m_{t}} \leq \frac{N}{m_{t}}.
	\end{equation*}
	Similarly, since $\rho_{ 1}\frac{m_{t}}{N}\|f_{t}^{z}\|_{K}^{2}$ is also less than or equals to the left hand side of the first inequality in this lemma, therefore, we have
	\begin{equation*}
	\rho_{ 1}\frac{m_{t}}{N}\|f_{t}^{z}\|_{K}^{2} \leq N,
	\quad \text{and} \quad		
	\|f_{t}^{z}\|_{K}\leq \frac{N}{\sqrt{\rho_{ 1}\cdot m_{t}}}.
	\end{equation*}
	Thus the result is shown.
\end{proof}

The second auxiliary lemma demonstrates the error of sampling method to be approximately of order $\frac{1}{\sqrt{N}}$.
This is so well known a result that perhaps every reader knows it.
However, since we failed to find a very good reference for it, a concise proof is given to show its correctness.

\begin{lemma}\label{lemma 7}
	For a sequence of bounded, independent, and identically distributed random variables $\{X_{i}\}$, there is
	\begin{equation*}
	\frac{\sum_{i=1}^{N} X_{i}}{N} - \mathbb{E}X_1= \mathcal{O} \left(N^{-\frac{1}{2}+\epsilon}\right)\quad \text{a.s.}
	\end{equation*}
	where $\epsilon$ is any positive constant.
\end{lemma}

\begin{proof}
	Let's denote $Z_{N} :=\frac{\sum_{i=1}^{N} X_{i}}{N} - \mathbb{E}X_1$.
	By Hoeffding's inequality,  for any positive constant $M$ there are relations
	\begin{equation*}
	P\left(|Z_{N}|\geq M  \right) \leq 2 e^{-2M^{2}\cdot N }
	\end{equation*}
	and
	\begin{equation*}
	P\left(|Z_{N}|\geq M \cdot N^{-1/2+\epsilon}  \right) \leq 2 e^{-2 M^{2}\cdot N^{\epsilon}}.
	\end{equation*}
	Since $e^{-2 M^{2}\cdot N^{\epsilon}}=o(\frac{1}{N^2})$, the summation $\sum_{N=1}^{\infty} P\left(|Z_{N}|\geq M \cdot N^{-1/2+\epsilon}  \right)$ convergences, and, therefore by Borelantelli lemma, we have
	\begin{equation*}
	\limsup_{N\to\infty} |Z_{N}| \leq M \cdot N^{-1/2+\epsilon} ~~~\text{a.s.}
	\end{equation*}
	which proves the lemma.
\end{proof}

The third auxiliary lemma shows the sampling error of classifier $f^{z}_{t}$.

\begin{lemma}\label{lemma 8}
	$$\mathbb { E }_{(\mathcal{X},\mathcal{Y})\sim P_{t} }[ \ell \left(\mathcal{Y}\cdot f^z_{t}( \mathcal{X} ) \right) ] - \sum_{ i = 1}^{m_{t}} \frac{\ell \left({y_{i,t}}\cdot f^{z}_{t}({\mathbf{x}_{i,t}} ) \right)}{m_{t}}  \leq \mathcal{O} \left(N^{-\frac{1}{4}+\epsilon}\right). $$
\end{lemma}

\begin{proof}
	We first denote some notations as follow:
	\begin{eqnarray}
	\tilde{\mathcal{ E}}_{t}(f) &:=& \mathbb { E }_{(\mathcal{X},\mathcal{Y})\sim P_{t} }[ \ell (\mathcal{Y}\cdot f^z_{t}( \mathcal{X} ) ) ] \notag \\
	\tilde{\mathcal{ E }}_{t}^{z}(f) &:=& \sum_{ i = 1}^{m_{t}} \frac{\ell \left({y_{i,t}}\cdot f^{z}_{t}({\mathbf{x}_{i,t}} ) \right)}{m_{t}} \notag \\
	\mathcal{B}_{R} &:=& \left\{ f\in \mathcal{ H }_{K} ~|~ \|f\|_{K} \leq R  \right\} \notag
	\end{eqnarray}
	Let $R^{*} =\frac{N}{\sqrt{\rho_{1}m_{t}}}$, then by \textit{Lemma \ref{lemma 6}} $f^{z}_{t}\in B_{R^*}$. Together with the fact that $\tilde{\mathcal{ E }}_{t}^{z}(f_{z})\leq \frac{N}{m_{t}}$ (c.f \textit{Theorem \ref{lemma 6}}), for any $\alpha>0$ we have
	\begin{eqnarray}
	&&P\left\{
	\tilde{\mathcal{ E}}_{t}(f^{z}_{t}) - 	\tilde{\mathcal{ E }}_{t}^{z}(f^{z}_{t}) > 4\alpha\left(1+\frac{N}{m_{t}}\right) ~\Big|~
	m_{t}		
	\right\} \notag \\
	&\leq&
	P\left\{ \sup_{f\in\mathcal{B}_{R^{*}}}\frac{\tilde{\mathcal{ E}}_{t}(f) - 	\tilde{\mathcal{ E }}_{t}^{z}(f)}{1+	\tilde{\mathcal{ E }}_{t}^{z}(f)} > 4\alpha~\Big|~
	m_{t}		
	\right\}\notag\\
	&\leq& \mathcal{N}\left(\frac{\alpha}{R^{*}}\right) \exp\left\{ -\frac{m_{t}\cdot\alpha^{2}}{32(1+\kappa R^{*})}\right\} \label{inequality 1}
	\end{eqnarray}	
	where the last inequality follows from \text{lemma 5} in \cite{wu2006analysis}, and $\mathcal{N}\left(\frac{\alpha}{R^{*}}\right)$ is the covering number defined as the minimal number of balls with radius $\frac{\alpha}{R^{*}}$ to cover the unite ball in RKHS.
	According to \cite{zhou2002covering}, the covering number has the relation
	\begin{equation*}
	\log \mathcal{N} (\epsilon) \leq M\left(\log \frac{1}{\epsilon}\right)^{d+1}
	\end{equation*}
	where $M$ is a constant, and $d$ is the dimension of the data set space.
	By plugging above inequality and $\alpha=\alpha^{*}(N,m_{t})=N^{-\epsilon}\cdot\sqrt{\frac{1+\kappa R^{*}}{m_{t}}}$ into \eqref{inequality 1}, we have
	\begin{eqnarray}
	&&P\left\{
	\tilde{\mathcal{ E}}_{t}(f^{z}_{t}) - 	\tilde{\mathcal{ E }}_{t}^{z}(f^{z}_{t}) > 4\left(1+\frac{N}{m_{t}}\right)\cdot\alpha^{*}(N,m_{t}) ~\Big|~
	m_{t}		
	\right\} \notag \\
	&\leq&
	\exp\left\{  M\left(\log \frac{N}{N^{-\epsilon}\sqrt{\rho_{ 1}+\rho_{ 1}\kappa R^{*}}}\right)^{d+1} -  \frac{1}{32} N^{-2\epsilon} \right\} \notag \\
	&\leq& \exp\left\{  M\left(\log \frac{N^{1-\epsilon}}{\sqrt{\rho_{ 1}}}\right)^{d+1} -  \frac{1}{32} N^{-2\epsilon} \right\} \notag \\
	\end{eqnarray}
	where $\epsilon$ is arbitrary positive constant.
	Furthermore, by above inequality and some easy calculations we can obtain
	\begin{align}
	&P\left\{
	\tilde{\mathcal{ E}}_{t}(f^{z}_{t}) - 	\tilde{\mathcal{ E }}_{t}^{z}(f^{z}_{t}) > 4\left(1+\frac{N}{m_{t}}\right)\cdot\alpha^{*}(N,m_{t}) ~\Big|~
	m_{t}		
	\right\} \notag\\
	&=  \mathbb{ E }_{m_{t}} \left[	P\left\{
	\tilde{\mathcal{ E}}_{t}(f^{z}_{t}) - 	\tilde{\mathcal{ E }}_{t}^{z}(f^{z}_{t}) > 4\alpha\left(1+\frac{N}{m_{t}}\right) ~\Big|~
	m_{t}		
	\right\} \right]\notag\\
	&\leq  \mathbb{ E }_{m_{t}} \left[ \exp\left\{  M\left(\log \frac{N^{1-\epsilon}}{\sqrt{\rho_{ 1}}}\right)^{d+1} -  \frac{1}{32} N^{-2\epsilon} \right\} \right]\notag\\
	&=  \exp\left\{  M\left(\log \frac{N^{1-\epsilon}}{\sqrt{\rho_{ 1}}}\right)^{d+1} -  \frac{1}{32} N^{-2\epsilon} \right\} \notag\\
	&= o \left(N^{-2}\right), \notag
	\end{align}
	and, therefore,
	\begin{equation*}
	\begin{aligned}
	&\sum_{N=1}^{\infty}P\left\{
	\tilde{\mathcal{ E}}_{t}(f^{z}_{t}) - 	\tilde{\mathcal{ E }}_{t}^{z}(f^{z}_{t}) > 4\left(1+\frac{N}{m_{t}}\right) \right.  \left. \cdot\alpha^{*}(N,m_{t}) ~\Big|~
	m_{t}		
	\right\} \\
	&< +\infty.
	\end{aligned}
	\end{equation*}
	
	By Borel-Cantelli Lemma, we have the relation
	\begin{equation*}
	\limsup_{N\to\infty} \tilde{\mathcal{ E}}_{t}(f^{z}_{t}) - 	\tilde{\mathcal{ E }}_{t}^{z}(f^{z}_{t}) - 4\left(1+\frac{N}{m_{t}}\right) \cdot\alpha^{*}(N,m_{t}) \leq 0 \quad \text{a.s.}.
	\end{equation*}
	Besides, according to the definition of $\alpha^{*}(\cdot,\cdot)$ and the strong law of large numbers, we have $\left(1+\frac{N}{m_{t}}\right) \cdot\alpha^{*}(N,m_{t})= \mathcal{O} \left(N^{-1/4+\epsilon}\right)$.
	Thus, $\tilde{\mathcal{ E}}_{t}(f^{z}_{t}) - 	\tilde{\mathcal{ E }}_{t}^{z}(f^{z}_{t}) \leq \mathcal{O} \left(N^{-1/4+\epsilon}\right)$
	which proves the result.
\end{proof}

With these preparation, we can prove \textit{Lemma \ref{lemma3}}.

\begin{proof}[{\bf The proof of Lemma 3}]
	We first consider the first part of sample error,  $\mathcal { E } \left(f ^z_ { 1}, \cdots ,f^z _ { T } \right) - \mathcal { E } _ { z } \left( f ^z_ { 1} ,\cdots ,f^z _ { T }\right)$. By the definition and some easy calculations, we can obtain,
	\begin{align}
	&\mathcal { E } \left(f ^z_ { 1}, \cdots ,f^z _ { T } \right) - \mathcal { E } _ { z } \left( f ^z_ { 1} ,\cdots ,f^z _ { T }\right) \notag \\
	=& \sum_{t=1}^{T} p(t)\mathbb { E }_{(\mathcal{X},\mathcal{Y})\sim P_{t} }[ \ell \left(\mathcal{Y}\cdot f^z_{t}( \mathcal{X} ) \right) ] \notag \\
	&- \sum_{t=1}^{T} \frac{m_{t}}{N}\sum_{ i = 1}^{m_{t}} \frac{\ell \left({y_{i,t}}\cdot f^{z}_{t}({\mathbf{x}_{i,t}} ) \right)}{m_{t}}  \notag \\
	=& \sum_{t=1}^{T} p(t)\mathbb { E }_{(\mathcal{X},\mathcal{Y})\sim P_{t} }[ \ell \left(\mathcal{Y}\cdot f^z_{t}( \mathcal{X} ) \right) ] \notag \\
	&- \sum_{t=1}^{T} p(t)\sum_{ i = 1}^{m_{t}} \frac{\ell \left({y_{i,t}}\cdot f^{z}_{t}({\mathbf{x}_{i,t}} ) \right)}{m_{t}} \notag \\
	&  +\sum_{t=1}^{T} \left(p(t)-\frac{m_{t}}{N}\right)\sum_{ i = 1}^{m_{t}} \frac{\ell \left({y_{i,t}}\cdot f^{z}_{t}({\mathbf{x}_{i,t}} ) \right)}{m_{t}} \notag\\
	=&  \sum_{t=1}^{T} p(t) \Bigg(\mathbb { E }_{(\mathcal{X},\mathcal{Y})\sim P_{t} }[ \ell \left(\mathcal{Y}\cdot f^z_{t}( \mathcal{X} ) \right) ] \left. -\sum_{ i = 1}^{m_{t}} \frac{\ell \left({y_{i,t}}\cdot f^{z}_{t}({\mathbf{x}_{i,t}} ) \right)}{m_{t}} \right) \notag \\
	&+ \sum_{t=1}^{T}\left[ \left(p(t)-\frac{m_{t}}{N}\right) \sum_{ i = 1}^{m_{t}} \frac{\ell \left({y_{i,t}}\cdot f^{z}_{t}({\mathbf{x}_{i,t}} ) \right)}{m_{t}} \right]  \label{eq. lemma proof auxilliary 1} 
	\end{align}
	Since, for each task $t$, the random variable $m_{t}$ is a counting process with bounded and independent increments, therefore, by \textit{Lemma \ref{lemma 7}}, there holds
	\begin{equation}{\label{eq. lemma proof auxiliarry 2}}
	p(t)-\frac{m_{t}}{N} =  \mathbb{E} \left(\frac{m_{t}}{N}\right) - \frac{m_{t}}{N} = \mathcal{O} \left(N^{-\frac{1}{2}+\epsilon}\right) \quad \text{a.s.}.
	\end{equation}
	By plugging above relation and \textit{lemma \ref{lemma 8}} into the last two lines of \eqref{eq. lemma proof auxilliary 1}, this relation leads to
	\begin{align} {\label{eq. lemma3 proof 1}}
	&\mathcal { E } \left(f ^z_ { 1}, \cdots ,f^z _ { T } \right) - \mathcal { E } _ { z } \left( f ^z_ { 1} ,\cdots ,f^z _ { T }\right)
	=  \mathcal{O} \left(N^{-\frac{1}{4}+\epsilon}\right)\quad \text{a.s.}.
	\end{align}
	
	Then we consider the second part of sample error, $\mathcal { E } _ { z } \left( f^{\mathcal { H }}_{1}, \cdots, f^{\mathcal { H }}_{T} \right) - \mathcal { E } \left( f^{\mathcal { H }}_{1}, \cdots, f^{\mathcal { H }}_{T} \right) $.
	Still by the definition and some easy calculations, we can obtain,
	\begin{align}
	& \mathcal { E } _ { z } \left( f^{\mathcal { H }}_{1}, \cdots, f^{\mathcal { H }}_{T} \right) - \mathcal { E } \left( f^{\mathcal { H }}_{1}, \cdots, f^{\mathcal { H }}_{T} \right) \notag \\
	=&\sum_{t=1}^{T} \frac{m_{t}}{N}\sum_{ i = 1}^{m_{t}} \frac{\ell \left({y_{i,t}}\cdot f^{\mathcal{ H }}_{t}({\mathbf{x}_{i,t}} ) \right)}{m_{t}} \notag \\
	&- \sum_{t=1}^{T} p(t)\mathbb { E }_{(\mathcal{X},\mathcal{Y})\sim P_{t} }[ \ell \left(\mathcal{Y}\cdot f^\mathcal{H}_{t}( \mathcal{X} ) \right) ] \notag \\
	=& \sum_{t=1}^{T} \left[ \left(\frac{m_{t}}{N}-p(t)\right)\sum_{ i = 1}^{m_{t}} \frac{\ell \left({y_{i,t}}\cdot f^{\mathcal{ H }}_{t}({\mathbf{x}_{i,t}} ) \right)}{m_{t}} \right] \notag \\
	&  +\sum_{t=1}^{T} p(t) \left[\sum_{ i = 1}^{m_{t}} \frac{\ell \left({y_{i,t}}\cdot f^{\mathcal{ H }}_{t}({\mathbf{x}_{i,t}} ) \right)}{m_{t}} \right. \notag\\
	& - \mathbb { E }_{(\mathcal{X},\mathcal{Y})\sim P_{t} }[ \ell \left(\mathcal{Y}\cdot f^\mathcal{H}_{t}( \mathcal{X} ) \right) ] \Bigg ]  \label{eq. lemma3 proof auxillary 3}
	\end{align}
	Moreover, since $\ell \left({y_{i,t}}\cdot f^{\mathcal{ H }}_{t}({\mathbf{x}_{i,t}} )\right) =\max\{0, 1-{y_{i,t}}f^{\mathcal{ H }}_{t}( {\mathbf{x}_{i,t}} )\}  \leq 1+\|f^{\mathcal{ H }}_{t}\|_{\infty} \leq 1+\kappa\|f^{\mathcal{ H }}_{t}\|_{K}$ and $\|f^{\mathcal{ H }}_{t}\|_{K}$ is finite, for each data sample $\ell \left({y_{i,t}}\cdot f^{\mathcal{ H }}_{t}({\mathbf{x}_{i,t}} )\right) $ is uniformly bounded by a constant $R$ where $R=1+\kappa\|f^{\mathcal{ H }}_{t}\|_{K}$.
	Therefore, by \eqref{eq. lemma proof auxiliarry 2} and the uniform boundedness of $\ell \left({y_{i,t}}\cdot f^{\mathcal{ H }}_{t}({\mathbf{x}_{i,t}} )\right)$, there holds
	\begin{equation}{\label{eq. lemma3 proof auxillary 4}}
	\begin{aligned}
	&\sum_{t=1}^{T} \left[ \left(\frac{m_{t}}{N}-p(t)\right)\sum_{ i = 1}^{m_{t}} \frac{\ell \left({y_{i,t}}\cdot f^{\mathcal{ H }}_{t}({\mathbf{x}_{i,t}} ) \right)}{m_{t}} \right]\\
	&\leq \sum_{t=1}^{T} \left[ \left(\frac{m_{t}}{N}-p(t)\right)\sum_{ i = 1}^{m_{t}} \frac{R }{m_{t}} \right]
	= \mathcal{O} \left(N^{-\frac{1}{2}+\epsilon}\right)\quad \text{a.s.}.
	\end{aligned}
	\end{equation}
	Besides, due to i.i.d. property of sampled data $(\mathbf{x}_{i,t},y_{i,t})$ for each task $t$, the random variable $\ell \left({y_{i,t}}\cdot f^{\mathcal{ H }}_{t}({\mathbf{x}_{i,t}} )\right) $ is also independent and identically distributed.
	Together with the boundedness of $\ell \left({y_{i,t}}\cdot f^{\mathcal{ H }}_{t}({\mathbf{x}_{i,t}} )\right) $, we can apply \textit{Lemma \ref{lemma 7}} to $\ell \left({y_{i,t}}\cdot f^{\mathcal{ H }}_{t}({\mathbf{x}_{i,t}} )\right) $ and obtain
	\begin{equation*}
	\begin{aligned}
	&\sum_{ i = 1}^{m_{t}} \frac{\ell \left({y_{i,t}}\cdot f^{\mathcal{ H }}_{t}({\mathbf{x}_{i,t}} ) \right)}{m_{t}}
	- \mathbb { E }_{(\mathcal{X},\mathcal{Y})\sim P_{t} }[ \ell \left(\mathcal{Y}\cdot f^\mathcal{H}_{t}( \mathcal{X} ) \right) ]\\
	&=\mathcal{O} \left(N^{-\frac{1}{2}+\epsilon}\right)\quad \text{a.s.}.
	\end{aligned}
	\end{equation*}
	By plugging this equation and \eqref{eq. lemma3 proof auxillary 4} into the last two lines of \eqref{eq. lemma3 proof auxillary 3}, we obtain
	\begin{equation}{\label{eq. lemma3 proof eq2}}
	\mathcal { E } _ { z } \left( f^{\mathcal { H }}_{1}, \cdots, f^{\mathcal { H }}_{T} \right) - \mathcal { E } \left( f^{\mathcal { H }}_{1}, \cdots, f^{\mathcal { H }}_{T} \right) = \mathcal{O} \left(N^{-\frac{1}{2}+\epsilon}\right)\quad \text{a.s.}.
	\end{equation}
	
	Finally combining \eqref{eq. lemma3 proof 1} and \eqref{eq. lemma3 proof eq2}, we have
	\begin{eqnarray}
	\begin{aligned}
	\mathcal { S } (N,T)&=\left\{ \mathcal { E } \left(f ^z_ { 1}, \cdots ,f^z _ { T } \right) - \mathcal { E } _ { z } \left( f ^z_ { 1} ,\cdots ,f^z _ { T }\right) \right\} \notag \\
	&\quad+ \left\{ \mathcal { E } _ { z } \left( f^{\mathcal { H }}_{1}, \cdots, f^{\mathcal { H }}_{T} \right) - \mathcal { E } \left( f^{\mathcal { H }}_{1}, \cdots, f^{\mathcal { H }}_{T} \right) \right\} \notag \\
	&= \mathcal{O} \left(N^{-\frac{1}{4}+\epsilon}\right)\quad \text{a.s.}. \notag
	\end{aligned}
	\end{eqnarray}
	which proves this lemma.
\end{proof}

\begin{proof}[{\bf The proof of Lemma 4}]
	First, we note that function $\ell(x)=\max\{0,1-x\}$ is a Lipschitz function, satisfying $|\ell(x)-\ell(y)|\leq |x-y|$ where $x$ and $y$ are arbitrary real numbers.
	Therefore, for any $h_{1},~h_{2},\dots,h_{T} \in \mathcal{ H }_{K}$ there is
	\begin{eqnarray}
	\begin{aligned}
	&\mathcal { E } \left(h_1, \cdots, h_{T} \right) - \mathcal { E } \left( f^{*}_{1}, \cdots, f^{*}_{T} \right)\notag \\
	&=
	\sum_{i=1}^{T}p(t) \mathbb { E }_{(\mathcal{X},\mathcal{Y})\sim P_{t} }[ \ell (\mathcal{Y}\cdot h_{t}( \mathcal{X} ) ) ] \notag \\
	&\quad-\sum_{i=1}^{T}p(t) \mathbb { E }_{(\mathcal{X},\mathcal{Y})\sim P_{t} }[ \ell (\mathcal{Y}\cdot f^{*}_{t}( \mathcal{X} ) ) ] \notag\\
	&\leq \sum_{i=1}^{T}p(t) \mathbb{ E } _{(\mathcal{X},\mathcal{Y})\sim P_{t} } \left|  h_{t}( \mathcal{X} ) ) -f^*_{t}( \mathcal{X} ) ) \right|  \notag \\
	&\leq \sum_{i=1}^{T}p(t) \left(\mathbb{ E } _{(\mathcal{X},\mathcal{Y})\sim P_{t} } \left|  h_{t}( \mathcal{X} ) ) -f^*_{t}( \mathcal{X} ) ) \right|^{2} \right)^{1/2}  \notag
	\end{aligned}
	\end{eqnarray}
	where the first inequality follows from Lipschitz continuity and the second one follows from Holder inequality.
	For simplicity, we denote $\|h_{t}-f_{t}\|_{\mathcal{L}^{2}_{P_{t}}}:= \left(\mathbb{ E } _{(\mathcal{X},\mathcal{Y})\sim P_{t} } \left|  h_{t}( \mathcal{X} ) ) -f^*_{t}( \mathcal{X} ) ) \right|^{2} \right)^{1/2}$.
	Through plugging above the last inequality to $D(N,T,\rho_{ 1},\rho_{ 2})$, we have
	\begin{eqnarray}
	\begin{aligned}
	&D(N,T,\rho_{ 1},\rho_{ 2}) \notag \\
	\leq&
	\inf_{ h _ { 1}, \cdots ,h _ { T }\in\mathcal { H } _{K} }
	\left\{
	\sum_{ t = 1}^{T} p(t) \|h_{t}-f^{*}_{t}\|_{\mathcal{L}^{2}_{P_{t}}} \notag +\frac{\rho _ { 1}}{N} \sum _ { t = 1} ^ { T }p(t)\| f^{*}_ { t } \|_{K} ^ { 2} \notag \right. \\
	& \left. + \frac{\rho _ { 2}}{N}\sum _ { t = 1} ^ { T } \left\| h_ { t } -\sum _ { s = 1} ^ { T }p(s) h _ { s } \right\|_{K}^ { 2}
	\right\} \notag   	 \\
	\leq& \inf_{R>0} \left\{
	\inf_{ \tiny\begin{array}{c}
		h _ { 1}, \cdots ,h _ { T }\in\mathcal { H } _{K}\\
		\|h_{t}\|_{K}\leq R
		\end{array} }
	\left\{
	\sum_{ t = 1}^{T} p(t) \|h_{t}-f^{*}_{t}\|_{\mathcal{L}^{2}_{P_{t}}, } 	 	 \right\} \notag \right. \\
	&+\frac{\rho _ { 1}  R^{2}}{N}+ \frac{2\rho _ { 2}R^{2}}{N}
	\Bigg\} \notag \\
	\leq& \inf_{R>0} \left\{
	C_{0}C_{s}\left(\log R \right)^{-s/4}
	+\frac{\rho _ { 1}  R^{2}}{N}+ \frac{2\rho _ { 2}R^{2}}{N}
	\right\} \notag
	\end{aligned}
	\end{eqnarray}
	where the last inequality follows immediately from \cite{smale2003estimating,zhou2013density}.
	In above, $C_{0}$ and $s$ are two positive constants, while $C_{s}$ is another constant depends on $s$.
	By choosing $R=N^{1/2-\epsilon}$ where $\epsilon>0$, we have
	\begin{equation*}
	\begin{aligned}
	&D(N,T,\rho_{ 1},\rho_{ 2}) \\
	\leq& C_{0}C_{s}\left( (1/2-\epsilon)\log N \right)^{-s/4}
	+{\rho _ { 1}  R^{-\epsilon}}+{2\rho _ { 2}R^{-\epsilon}}\\
	=& \mathcal{O} \left(\log N\right)^{-s/4}.
	\end{aligned}
	\end{equation*}
	which proves the result.
\end{proof}

\begin{proof}[{\bf The proof of Lemma 5}]
	By the definition and some easy calculation, we get
	\begin{eqnarray}
	\begin{aligned}
	&\mathcal{F}(N,T,\rho_1, \rho_2)\\
	=&
	-\frac{\rho_{ 1}}{N}\sum_{ i = 1}^{T} \left(p(t)-\frac{m_{t}}{N}\right)\| f^{\mathcal{H}}_ { t } \|_{K} ^ { 2}  \notag\\
	&+\frac{2\rho _ {2}}{N}\left\| \sum _ { t = 1} ^ { T } \left( p(t) - \frac{m_t}{N}\right) f^{\mathcal{H}}_{t} \right\|_{K} \sum _ { t = 1} ^ { T } \left\| f^{\mathcal{H}}_ { t } -\sum _ { s = 1} ^ { T }\frac{m_s}{N} f^{\mathcal{H}}_{s} \right\|_{K}\notag\\
	\leq& -\frac{\rho_{ 1}}{N}\sum_{ i = 1}^{T} \left(p(t)-\frac{m_{t}}{N}\right)\| f^{\mathcal{H}}_ { t } \|_{K} ^ { 2}  \notag\\
	&+\frac{2\rho _ {2}}{N} \left(\sum_{t=1}^{T} \left\|p(t)-\frac{m_{t}}{N}\right\|_{K} \cdot \left\|f_{t}^\mathcal{H}\right\|_{K} \right) \\
	&\cdot	\left( \sum _ { t = 1} ^ { T } \left\| f^{\mathcal{H}}_ { t } \right\|_{K}+ \sum _ { s = 1} ^ { T } T \left\| f^{\mathcal{H}}_{s} \right\|_{K} \right)\notag
	\end{aligned}
	\end{eqnarray}
	where the terms $\frac{\rho_1}{N}$, $\frac{\rho_2}{N}$, and $p(t)-\frac{m_{t}}{N}$ determine the convergent rate of the frequency error.
	Since, for each task $t$, the random variable $m_{t}$ is a counting process with bounded and independent increments, therefore, by \textit{Lemma \ref{lemma 7}}, there holds
	\begin{equation*}
	p(t)-\frac{m_{t}}{N} =  \mathbb{E} \left(\frac{m_{t}}{N}\right) - \frac{m_{t}}{N} = \mathcal{O} \left(N^{-\frac{1}{2}+\epsilon}\right) \quad \text{a.s.}.
	\end{equation*}
	Therefore, $\mathcal{F}(T,N,\rho_1, \rho_2)\leq  \mathcal{O} \left(N^{-\frac{3}{2}+\epsilon}\right) ~ \text{a.s.}$, which shows the result.
\end{proof}





